\documentclass[12pt]{article}

\usepackage{microtype}
\usepackage{graphicx}
\usepackage{subfigure}
\usepackage{booktabs} 
\usepackage{cleveref}

\usepackage{hyperref}

\usepackage{amsmath}
\usepackage{amsfonts}
\usepackage{dsfont}
\usepackage{amssymb,amsthm,commath}
\usepackage{amsmath, amsthm, amssymb}
\usepackage{bm,xfrac,xcolor}
\usepackage{nccmath}
\usepackage[toc,page,header]{appendix}

\usepackage{algorithm}
\usepackage[noend]{algorithmic}
\usepackage{bbm}
\usepackage{enumitem}
\usepackage{gensymb}
\usepackage{refcount}
\usepackage{mathtools}
\usepackage{stackengine}
\newcommand\barbelow[1]{\stackunder[1.2pt]{$#1$}{\rule{1.5ex}{.1ex}}}
\newcommand*\samethanks[1][\value{footnote}]{\footnotemark[#1]}

\DeclarePairedDelimiter{\ceil}{\lceil}{\rceil}
\usepackage[top=1in,bottom=1in,left=1in,right=1in]{geometry}
\makeatletter
\def\myfnt{\ifx\protect\@typeset@protect\expandafter\footnote\else\expandafter\@gobble\fi}
\makeatother

\makeatletter
\def\BState{\State\hskip-\ALG@thistlm}
\makeatother





\hypersetup{colorlinks,urlcolor=black, citecolor=blue}

\def\1{\mathds{1}}

\newcommand{\TP}{\emph{TP}}
\newcommand{\TN}{\emph{TN}}

\newcommand{\TPs}{\emph{TP }}

\newcommand{\Ccal}{{\cal C}}

\newcommand{\Hcal}{{\cal H}}

\newcommand{\Scal}{{\cal S}}

\newcommand{\Xcal}{{\cal X}}
\newcommand{\Ycal}{{\cal Y}}

\newcommand{\Pmbb}{\mathbb{P}}

\newcommand{\Rmbb}{\mathbb{R}}

\newcommand{\Umbb}{\mathbb{U}}

\newcommand{\mmbf}{\mathbf{m}}

\newcommand{\oline}[1]{\mkern 1.5mu\overline{\mkern-1.5mu#1}}

\newcommand{\Cbar}{\oline{C}}

\renewcommand{\hbar}{\oline{h}}

\newcommand{\bmmbf}{\oline{\mathbf{m}}}
\newcommand{\bmone}{\oline{m}_{11}}
\newcommand{\bmzero}{\oline{m}_{00}}

\newcommand{\bell}{\oline{\ell}}
\newcommand{\btau}{\oline{\tau}}

\newcommand{\hmmbf}{\hat{\mathbf{m}}}

\newcommand{\hpone}{\hat{p}_{11}}
\newcommand{\hpzero}{\hat{p}_{00}}
\newcommand{\hqone}{\hat{q}_{11}}
\newcommand{\hqzero}{\hat{q}_{00}}
\newcommand{\hqnot}{\hat{q}_{0}}

\newcommand{\sphi}{\phi^*}
\newcommand{\smmbf}{\mathbf{m}^*}
\newcommand{\smone}{{m_{11}^*}}
\newcommand{\smzero}{{m_{00}^*}}
\newcommand{\spone}{{p_{11}^*}}
\newcommand{\spzero}{{p_{00}^*}}
\newcommand{\sqone}{{q_{11}^*}}
\newcommand{\sqzero}{{q_{00}^*}}
\newcommand{\sqnot}{{q_{0}^*}}

\newcommand{\tmmbf}{\barbelow{\mathbf{m}}}

\newcommand{\tell}{\barbelow{\ell}}
\newcommand{\ttau}{\barbelow{\tau}}

\newcommand{\pphi}{{\phi'}}

\newcommand{\ppone}{{p}_{11}'}
\newcommand{\ppzero}{{p}_{00}'}
\newcommand{\pqone}{{q}_{11}'}
\newcommand{\pqzero}{{q}_{00}'}
\newcommand{\pqnot}{{q}_{0}'}

\newcommand{\ppphi}{{\phi''}}

\newcommand{\pppone}{{p}_{11}''}
\newcommand{\pppzero}{{p}_{00}''}
\newcommand{\ppqone}{{q}_{11}''}
\newcommand{\ppqzero}{{q}_{00}''}
\newcommand{\ppqnot}{{q}_{0}''}

\newcommand{\hphi}{\hat{\phi}}

\newcommand{\Chat}{\hat{C}}

\newcommand{\hhat}{\hat{h}}

\DeclareMathOperator{\argmin}   {arg\,min}
\DeclareMathOperator{\argmax}   {arg\,max}

\newcommand{\baligned}     {\begin{aligned}}
	\newcommand{\ealigned}     {\end{aligned}}
\newcommand{\barray}       {\begin{array}}
	\newcommand{\earray}       {\end{array}}
\newcommand{\bbmatrix}     {\begin{bmatrix}}
	\newcommand{\ebmatrix}     {\end{bmatrix}}
\newcommand{\bcases}       {\begin{cases}}
	\newcommand{\ecases}       {\end{cases}}
\newcommand{\bcenter}      {\begin{center}}
	\newcommand{\ecenter}      {\end{center}}
\newcommand{\bcolumn}      {\begin{column}}
	\newcommand{\ecolumn}      {\end{column}}
\newcommand{\bcolumns}     {\begin{columns}}
	\newcommand{\ecolumns}     {\end{columns}}
\newcommand{\benumerate}   {\begin{enumerate}}
	\newcommand{\eenumerate}   {\end{enumerate}}
\newcommand{\bequation}    {\begin{equation}}
	\newcommand{\eequation}    {\end{equation}}
\newcommand{\bequationn}   {\begin{equation*}}
	\newcommand{\eequationn}   {\end{equation*}}
\newcommand{\bfigure}      {\begin{figure}}
	\newcommand{\efigure}      {\end{figure}}
\newcommand{\bflushright}  {\begin{flushright}}
	\newcommand{\eflushright}  {\end{flushright}}
\newcommand{\bitemize}     {\begin{itemize}}
	\newcommand{\eitemize}     {\end{itemize}}
\newcommand{\bpmatrix}     {\begin{pmatrix}}
	\newcommand{\epmatrix}     {\end{pmatrix}}
\newcommand{\bsubequations}{\begin{subequations}}
	\newcommand{\esubequations}{\end{subequations}}
\newcommand{\btable}       {\begin{table}}
	\newcommand{\etable}       {\end{table}}
\newcommand{\btabular}     {\begin{tabular}}
	\newcommand{\etabular}     {\end{tabular}}
\newcommand{\bvmatrix}     {\begin{vmatrix}}
	\newcommand{\evmatrix}     {\end{vmatrix}}

\newcommand{\bequali}      {\bsubequations\begin{align}}
	\newcommand{\eequali}      {\end{align}\esubequations}

\newtheorem{assumption}{Assumption}
\newtheorem{prop}{Proposition}
\newtheorem{example}{Example}
\newtheorem{remark}{Remark}
\newtheorem{theorem}{Theorem}
\newtheorem{corollary}{Corollary}
\newtheorem{definition}{Definition}
\newtheorem{lemma}{Lemma}

\newcommand{\balgorithm}  {\begin{algorithm}}
	\newcommand{\ealgorithm}  {\end{algorithm}}
\newcommand{\balgorithmic}{\begin{algorithmic}}
	\newcommand{\ealgorithmic}{\end{algorithmic}}
\newcommand{\bassumption} {\begin{assumption}}
	\newcommand{\eassumption} {\end{assumption}}
\newcommand{\bcorollary}  {\begin{corollary}}
	\newcommand{\ecorollary}  {\end{corollary}}
\newcommand{\bdefinition} {\begin{definition}}
	\newcommand{\edefinition} {\end{definition}}
\newcommand{\bexample}    {\begin{example}}
	\newcommand{\eexample}    {\end{example}}
\newcommand{\bprop}    {\begin{prop}}
	\newcommand{\eprop}    {\end{prop}}
\newcommand{\blemma}      {\begin{lemma}}
	\newcommand{\elemma}      {\end{lemma}}
\newcommand{\bproblem}    {\begin{problem}}
	\newcommand{\eproblem}    {\end{problem}}
\newcommand{\bproof}      {\begin{proof}}
	\newcommand{\eproof}      {\end{proof}}
\newcommand{\bremark}     {\begin{remark}}
	\newcommand{\eremark}     {\end{remark}}
\newcommand{\btheorem}    {\begin{theorem}}
	\newcommand{\etheorem}    {\end{theorem}}
\usepackage{placeins,afterpage}
\usepackage[customcolors]{hf-tikz}
\usepackage{adjustbox}
\hypersetup{colorlinks,urlcolor=blue}

\usetikzlibrary{calc}
\usepackage{chngcntr}
\counterwithout{equation}{section}
\counterwithout{table}{section}
\counterwithout{figure}{section}
\counterwithout{algorithm}{section}
\begin{document}

\title{Performance Metric Elicitation from Pairwise Classifier Comparisons}
\setcounter{footnote}{1}
\author{Gaurush Hiranandani\thanks{
University of Illinois Urbana-Champaign}\\
gaurush2@illinois.edu
\and
Shant Boodaghians\samethanks \\
boodagh2@illinois.edu
\and
Ruta Mehta\samethanks \\
rutameht@illinois.edu
\and
Oluwasanmi Koyejo\samethanks \\
sanmi@illinois.edu
}
\date{\today}

\flushbottom
\maketitle

\begin{abstract}
Given a binary prediction problem, which performance metric should the classifier optimize? 
We address this question by formalizing the problem of \emph{Metric Elicitation}. The goal of metric elicitation is to discover the performance metric of a practitioner, which reflects her innate rewards (costs) for correct (incorrect) classification. 
In particular, we focus on eliciting binary classification performance metrics from pairwise feedback, where a practitioner is queried to provide relative preference between two classifiers. By exploiting key geometric properties of the space of confusion matrices, we obtain provably query efficient algorithms for eliciting linear and linear-fractional performance metrics. We further show that our method is robust to feedback and finite sample noise.
\end{abstract}
\section{Introduction}
\label{ssec:introduction}

Selecting an appropriate performance metric is crucial to the real-world utility of predictive machine learning.  
Specialized teams of statisticians and economists are routinely hired in the industry to monitor many metrics -- since optimizing the wrong metric directly translates into lost revenue~\cite{Dmitriev2016MeasuringM}. 
Medical predictions are another important application, where ignoring cost sensitive trade-offs can directly impact lives~\cite{sox1988medical}. 
Unfortunately, there is scant formal guidance within the literature for how a practitioner/user might choose a metric, beyond a few common default choices~\cite{caruana2004data, ferri2009experimental, sokolova2009systematic}, and even less guidance on selecting a metric which reflects the preferences of the practitioners/users.

\begin{figure}[t]
	\centering 
		\includegraphics[scale=0.5]{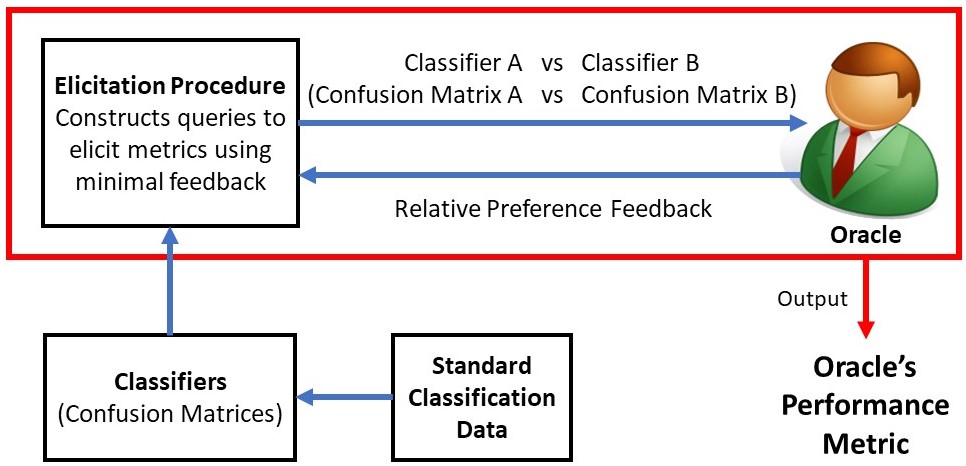}
	\caption{Metric Elicitation framework.}
	\label{fig:ME}
\end{figure}

\textbf{Metric Elicitation:} 
Motivated by the principle that the performance metric which best reflects implicit user tradeoffs results in learning models that best resonate with user preferences \cite{elkan2001foundations, sokolova2009systematic},   
we introduce a framework, \emph{metric elicitation (ME)}, for determining the binary classification performance metric from user feedback. Since human feedback is costly, the goal is to use as little feedback as  possible. 
On its face, ME simply requires querying a user (oracle) to determine the quality she assigns to classifiers that are learned from standard classification data; however, humans are often inaccurate in providing absolute preferences~\cite{qian2013active}. Therefore, we propose to employ pairwise comparison queries, where the user (oracle) is asked to compare two classifiers and provide an indicator of relative preference. 
Based on that relative preference feedback, we elicit the innate performance metric of the user (oracle). See Figure~\ref{fig:ME} for visual intuition of the framework.

Our approach is inspired by a large literature in economics and psychology on \emph{preference elicitation} \cite{samuelson1938note, mas1977recoverability, varian2005, braziunas2007minimax}. 
Here, the goal is to learn user preferences from purchases at posted prices.  
Since there is no notion of prices or purchases in {\em ME} for machine learning, standard approaches from these studies do not apply. In addition, we emphasize that the notion of pairwise classifier comparison is not new and is already prevalent in the industry. An example is A/B testing \cite{tamburrelli2014towards}, 
where the whole population of users acts as an oracle.\footnote{In A/B testing, sub-populations of users are shown classifier A vs. classifier B, and their responses determine the overall preference. Interestingly, while each person is shown a sample output from one of the classifiers, the entire user population acts as the oracle for comparing classifiers.}  Similarly, classifier comparison by a single expert is becoming commonplace due to advances in the field of interpretable machine learning~\cite{ribeiro2016should, doshi2017towards}.

In this first edition of ME, we focus on the most common performance metrics which are functions of the confusion matrix \cite{koyejo2014consistent, narasimhan2015consistent, sokolova2009systematic}, particularly, linear and ratio-of-linear functions.\footnote{Metrics depending on factors such as model complexity and interpretability are beyond the scope of this manuscript.} 
This includes almost all modern metrics such as
accuracy, $F_\beta$-Measure, Jaccard Similarity Coefficient~\cite{sokolova2009systematic}, etc.
By construction, pairwise classifier comparisons may be conceptually represented by their associated pairwise confusion matrix comparisons. 
Despite this apparent simplification, the problem remains challenging because one can only query feasible confusion matrices, i.e. confusion matrices for which there exists a classifier. As we show, our characterization of the space of confusion matrices enables the design of efficient binary-search type procedures that identify the innate performance metric of the oracle. 
While classifier (confusion matrix) comparisons may introduce additional noise, our approach remains robust, both to noise from classifier (confusion matrix) estimation, and to noise in the comparison itself. Thus, our work directly results in a practical algorithm. 

\textbf{Example:}  
Consider the case of cancer diagnosis, where a doctor's unknown, innate performance metric is a linear function of the confusion matrix, i.e., she has some innate reward values for True Positives and True Negatives 
-- 
equivalently (equiv.), costs for False Positives and False Negatives
-- 
based on known consequences of misdiagnosis. 
Here, the doctor takes the role of the oracle. 
Our proposed approach exploit the space of confusion matrices associated with all possible classifiers that can be learned from standard classification data
and determine the underlying rewards (equiv., costs) provably using the least possible number of pairwise comparison queries posed to the doctor. 


\noindent Our contributions are summarized as follows:
\begin{itemize}
	\item We propose the technical problem of \emph{Metric Elicitation}, a framework for determining supervised learning metrics from user feedback. For the case of pairwise feedback, we show that under certain conditions ME is equivalent to learning preferences between pairs of confusion matrices. 
	\item When the underlying metric is linear, we propose a binary search algorithm that can recover the metric with query complexity that decays logarithmically with the desired resolution. We further show that our query-complexity rates match the lower bound. 
	\item We extend the elicitation algorithm to more complex linear-fractional performance metrics.
	\item We prove robustness of the proposed approach under feedback and classifier estimation noise.
\end{itemize}
\section{Background}
\label{sec:background}

Let $X \in \Xcal$ and $Y \in \{0, 1\}$ represent the input and output random variables respectively (0 = negative class, 1 = positive class). 
We assume a dataset of size $n$, $\{(x_i, y_i)\}_{i=1}^n$, generated \emph{iid} from a data generating distribution $ \Pmbb \overset{\text{iid}}{\sim} (X, Y)$. 
Let $f_X$ be the marginal distribution for $\Xcal$. Let $\eta(x) = \Pmbb(Y = 1 | X = x)$ and $\zeta = \Pmbb(Y = 1)$ represent the conditional and the unconditional probability of the positive class, respectively. Note that the earlier term is a function of the input $x$; whereas, the latter is a constant. 
We denote a classifier by $h$, and let $\Hcal = \{h : \Xcal \rightarrow [0, 1]\}$ be the set of all classifiers. 
A confusion matrix for a classifier $h$ is denoted by $C(h, \Pmbb) \in \Rmbb^{2 \times 2}$, comprising true positives (TP), false positives (FP), false negatives (FN), and true negatives (TN) and is given by:
\begin{ceqn}
\begin{align}
	C_{11} &= TP(h, \Pmbb) = \Pmbb(Y = 1, h = 1), \nonumber \\
	C_{01} &= FP(h, \Pmbb) = \Pmbb(Y = 0, h = 1), \nonumber \\
	C_{10} &= FN(h, \Pmbb) = \Pmbb(Y = 1, h = 0), \nonumber \\ 
	C_{00} &= TN(h, \Pmbb) = \Pmbb(Y = 0, h = 0). 
	\label{eq:components}
\end{align}\nopagebreak
\end{ceqn}\nopagebreak
Clearly,  $\sum_{i, j} C_{ij} = 1$. We denote the set of all confusion matrices by $\Ccal = \{C(h, \Pmbb)  : h \in \Hcal\}$. 
Under the population law $\Pmbb$, the components of the confusion matrix can be further decomposed as: 
$FN(h, \Pmbb) = \zeta - TP(h, \Pmbb)$ and $FP(h, \Pmbb) = 1 - \zeta - TN(h, \Pmbb).$ 
This decomposition reduces the four dimensional space to two dimensional space. Therefore, the set of confusion matrices can be defined as $\Ccal = \{(TP(h, \Pmbb), TN(h, \Pmbb)) : h \in \Hcal\}$. For clarity, we will suppress the dependence on $\Pmbb$ in our notation. In addition, we will subsume the notation $h$ if it is implicit from the context and denote the confusion matrix by $C = (TP, TN)$. We represent the boundary of the set $\Ccal$ by $\partial\Ccal$. Any hyperplane (line) $\ell$ in the $(tp, tn)$ coordinate system is given by:
$$\ell := a\cdot tp + b\cdot tn = c, \quad \text{ where } a, b, c \in \Rmbb.$$
Let $\phi : [0, 1]^{2 \times 2}  \rightarrow \Rmbb$ be the performance metric for a classifier $h$ determined by its confusion matrix $C(h)$. Without loss of generality (WLOG), we assume that $\phi$ is a utility, so that larger values are better. 

\subsection{Types of Performance Metrics}
\label{ssec:metrics}

We consider two of the most common families of binary classification metrics, namely linear and linear-fractional functions of the confusion matrix \eqref{eq:components}.
\bdefinition
Linear Performance Metric (LPM): We denote this family by $\varphi_{LPM}$. Given constants (representing costs or weights) $\{a_{11}, a_{01}, a_{10}, a_{00}\} \in \Rmbb^4$, we define the metric as:
\begin{ceqn}
\begin{align}
\phi(C) &= a_{11}TP + a_{01}FP  + a_{10}FN + a_{00}TN \nonumber \\
				  &= m_{11}TP + m_{00}TN + m_0,   
\label{eq:linear}
\end{align}
\end{ceqn}
where $m_{11} = (a_{11} - a_{10})$, $m_{00} = (a_{00} - a_{01})$, and $m_0 = a_{10}\zeta + a_{01}(1-\zeta)$. 
\edefinition
\bexample Weighted Accuracy (WA)~\cite{steinwart2007compare}: 
$$WA = w_1TP+w_2TN,$$
where $w_1, w_2 \in [0, 1]$ 
($w_1, w_2$ can be shifted and scaled to $[0, 1]$ without changing the learning problem ~\cite{narasimhan2015consistent}).
\label{ex:lossbased}
\eexample

\bdefinition
Linear-Fractional Performance Metric (LFPM): We denote this family by $\varphi_{LFPM}$. Given constants (representing costs or weights) $\{a_{11}, a_{01}, a_{10}, a_{00}$, $b_{11}, b_{01}, b_{10}, b_{00}\} \in \Rmbb^8$, we define the metric as: 
\begin{ceqn}
\begin{align}
\phi(C) &= \frac{a_{11}TP +  a_{01}FP +  a_{10}FN +  a_{00}TN}{b_{11}TP +  b_{01}FP +  b_{10}FN +  b_{00}TN} \nonumber \\
&= \frac{p_{11}TP +  p_{00}TN +  p_0}{q_{11}TP +  q_{00}TN +  q_0},
\label{linear-fractional}
\end{align}
\end{ceqn}
where $p_{11} = (a_{11} - a_{10})$, $p_{00} = (a_{00} - a_{01})$, $q_{11} = (b_{11} - b_{10})$, $q_{00} = (b_{00} - b_{01})$, $p_0 = a_{10}\zeta + a_{01}(1 - \zeta)$, $q_0 = b_{10}\zeta + b_{01}(1 - \zeta)$.
\edefinition
\bexample

The $F_\beta$ measure and the Jaccard similarity coefficient (JAC)~\cite{sokolova2009systematic}:
\begin{align}
    F_{\beta} = \frac{TP}{\frac{TP}{1+\beta^2} - \frac{TN}{1+\beta^2} + \frac{\beta^2\zeta + 1 - \zeta}{1+\beta^2}}, \quad JAC = \frac{TP}{1-TN}.
\label{ex:lf-examples}
\end{align}
\eexample

\subsection{Bayes Optimal and Inverse Bayes Optimal Classifiers}
\label{ssec:bayes}

Given a performance metric $\phi$, the Bayes utility $\btau$ is the optimal value of the performance metric over all classifiers, i.e., $\btau = \sup_{h \in \Hcal}\phi(C(h)) = \sup_{C \in \Ccal}\phi(C)$. The Bayes classifier $\hbar$ (when it exists) is the classifier that optimizes the performance metric, so $\hbar = \argmax\limits_{h \in \Hcal}\phi(C(h)).$ Similarly, the Bayes confusion matrix is given by $\Cbar = \argmax\limits_{C \in \Ccal}\phi(C).$ 
We further define the inverse Bayes utility   
$\ttau = \inf_{h \in \Hcal}\phi(C(h)) = \inf_{C \in \Ccal}\phi(C)$. The inverse Bayes classifier is given by $\barbelow{h} = \argmin\limits_{h \in \Hcal}\phi(C(h))$. Similarly, the inverse Bayes confusion matrix is given by $\barbelow{C} = \argmin\limits_{C \in \Ccal}\phi(C).$ 
Notice that for $\phi \in \varphi_{LPM}$ \eqref{eq:linear}, 
the Bayes classifier predicts the label which  maximizes the expected utility conditioned on the instance, as discussed below.
\bprop Let $\phi \in \varphi_{LPM}$, then 
$$
\hbar(x) = \left\{\begin{array}{lr}
			 \1[\eta(x) \geq \frac{m_{00}}{m_{11} + m_{00}}],& \; m_{11} + m_{00} \geq 0 \\
			 \1[\frac{m_{00}}{m_{11} + m_{00}} \geq \eta(x) ],& \;  o.w. 
	 	 \end{array}\right\}
$$
is a Bayes optimal classifier \emph{w.r.t} 
$\phi$. Further, the inverse Bayes classifier is given by $\barbelow{h}= 1 - \hbar$.
\label{pr:bayeslinear}
\eprop

\subsection{Problem Setup}
\label{ssec:query}

We first formalize \emph{oracle query}. 
Recall that by the definition of confusion matrices~\eqref{eq:components}, there exists a surjective mapping from $\Hcal \rightarrow \Ccal$.
An oracle is queried to determine relative preference between two classifiers. 
However, since we only consider metrics which are functions of the confusion matrix, a comparison query over classifiers becomes equivalent to a comparison query over confusion matrices in our setting.

\bdefinition
Oracle Query: Given two classifiers $h, h'$ (equiv. to confusion matrices $C, C'$ respectively), a query to the Oracle (with metric $\phi$) is represented by:
\begin{ceqn}
\begin{align}
\Gamma(h, h') = \Omega(C, C') &= \1[\phi(C) > \phi(C')] =: \1[C \succ C'],
\end{align}
\end{ceqn}
where $\Gamma: \Hcal \times \Hcal \rightarrow \{0,1\}$ and $\Omega: \Ccal \times \Ccal \rightarrow \{0, 1\}$. The query denotes whether $h$ is preferred to $h'$ (equiv. to $C$ is preferred to $C'$) as measured according to $\phi$.
\label{def:query}
\edefinition

We emphasize that depending on practical convenience, the oracle may be asked to compare either confusion matrices or classifiers achieving the corresponding confusion matrices, via approaches discussed in Section~\ref{ssec:introduction}. 
Henceforth, for simplicity of notation, we will treat any comparison query as confusion matrix comparison query. 
Next, we state the metric elicitation problem. 

\bdefinition 
    Metric Elicitation (given $\Pmbb$): Suppose that the oracle's true, unknown performance metric is $\phi$.  Recover a metric $\hphi$ by querying the oracle for as few pairwise comparisons of the form $\Omega(C, C')$, such that $\Vert\phi - \hphi\Vert_{\_\_} < \kappa$ for sufficiently small $\Rmbb \ni\kappa > 0$ and for any suitable norm $\Vert \cdot \Vert_{\_\_}$.
\label{eq:me}
\edefinition
Notice that Definition~\ref{eq:me} involves true population quantities $C, C'$ (See \eqref{eq:components}). However, in practice, we are given only finite samples. This leads to a more practical definition of metric elicitation problem.

\bdefinition 
Metric Elicitation (given $\{(x_i,y_i)\}_{i=1}^n$): The same problem as stated in Definition~\ref{eq:me}, except that the queries are of the form $\Omega(\Chat, \Chat')$, where $\Chat, \Chat'$ are the estimated confusion matrices from the samples.
\label{eq:mefinite}
\edefinition

Ultimately, we want to perform ME as described in  Definition~\ref{eq:mefinite}. A good approach to do so is to first solve ME as defined in Definition~\ref{eq:me}, i.e, ME assuming access to the appropriate population quantities, and then consider practical implementation using finite data. This is a standard approach in decision theory (see e.g. \cite{koyejo2015consistent}), where estimation error from finite samples is adjudged as a noise source and handled accordingly.
\section{Confusion Matrices}
\label{sec:confusion}

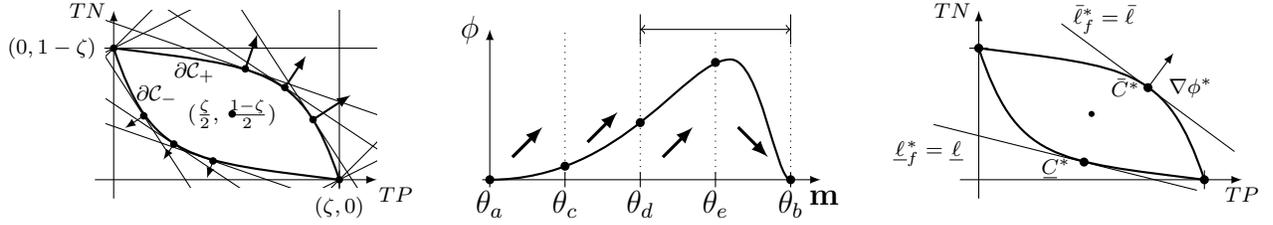
\begin{figure*}[t]
	\centering
	\begin{tikzpicture}[scale = 1.0]
    

    	\begin{scope}[shift={(-5,0)},scale = 0.5]\scriptsize
    	\def\r{0.1};
    	\def\s{0.06};
	
	\draw[thick] (0,3.5) .. controls (4,3) and (5,3) .. (6,0)
    	.. controls (2,0.5) and (1,0.5) .. (0,3.5);
    \draw[-latex] (0,-.5)--(0,4.5); 
    \draw[-latex] (-.5,0)--(7,0);
    \node[left] at (0,4.5) {$TN$};
    \node[below] at (7.5,0) {$TP$};
    \draw (6,0) +(0,0.25) -- +(0,-.25);
    \draw (0,3.5) +(.25,0) -- +(-.25,0);
    \node[below] at (6,-.25) {$(\zeta, 0)$};
    \node[left] at (-0.25,3.5) {$(0, 1-\zeta)$};
    
    \coordinate (C*) at (4.55,2.455);    
    \coordinate (C1) at (3.5,2.95);
    \coordinate (C2) at (5.3,1.6);
    
	\coordinate (Cent) at (3.15,1.75);    
    
    \coordinate (Ct) at (1.6,0.95);
    \coordinate (Ct1) at (2.64,0.5);
    \coordinate (Ct2) at (0.8,1.7); 
    \coordinate (C+) at (2.15,2.85);
    
    \coordinate (C-) at (1.15,2.25);

    \fill[color=black] 
    		(0,3.5) circle (\r)
    		(6,0) circle (\r)
        (Cent) circle (\r)
        (C*) circle (\r)
        (C1) circle (\r)
        (C2) circle (\r)
        (Ct) circle (\r)
        (Ct1) circle (\r)
        (Ct2) circle (\r);
    
    	\clip (-0.2,-0.2) rectangle (7,4.5);   
    
    \draw (C*) +(-24,16) -- +(24,-16);
    \draw[-latex, thick] (C*) -- +(.56,0.84);
    \draw (C1) +(-23,8) -- +(23,-8);
    \draw[-latex,thick] (C1) -- +(.32,0.92);
    \draw (C2) +(-13,20) -- +(13,-20);
    \draw[-latex,thick] (C2) -- +(1,.65);
    
    \draw (Ct) +(-24,16) -- +(24,-16);
    \draw[-latex] (Ct) -- +(-.28,-0.42);
    \draw (Ct1) +(-23,8) -- +(23,-8);
    \draw[-latex] (Ct1) -- +(-.16,-0.46);
    \draw (Ct2) +(-13,20) -- +(13,-20);
    \draw[-latex] (Ct2) -- +(-.5,-.325);
    
	
	\draw(6,0) +(0,10) -- +(0,-20);
	\draw(6,0) +(10,10) -- +(-20,-20);
	\draw(6,0) +(20,10) -- +(-20,-10);
	\draw(0,3.5) +(10,0) -- +(-20,0);    
	\draw(0,3.5) +(10,10) -- +(-20,-20);    
	\draw(0,3.5) +(10,5) -- +(-20,-10);    
    
    \node at (Cent) {{$(\frac{\zeta}{2},\, \frac{1-\zeta}{2})$}};
    
    \node at (C+) {{$\partial\Ccal_+$}};
    \node at (C-) {{$\partial\Ccal_-$}};

    \end{scope}

    
	\def\r{0.06};
	
    \draw[thick] (0,0) .. controls (1.8,0) and (2.6,1.6) .. (3.2,1.6) 
    ..controls (3.6,1.6) and (3.8,0) .. (4,0);
    \draw[-latex] (0,-.1)--(0,2); 
    \draw[-latex] (-0.1,0)--(4.4,0);
    \node[left] at (0,2) {$\phi$};
    \node[below right] at (4.1,0) {$\mathbf m$};
   
   	\coordinate (C1) at (0,0.00);
    \coordinate (C2) at (1,0.18);
    \coordinate (C3) at (2,0.76);
    \coordinate (C4) at (3,1.56);
    \coordinate (C5) at (4,0.00);
    
    \node[below] at (0,0) {$\theta_a$};
    \node[below] at (1,0) {$\theta_c$};
    \node[below] at (2,0) {$\theta_d$};
    \node[below] at (3,0) {$\theta_e$};
    \node[below] at (4,0) {$\theta_b$};
    
    \foreach \x in {1,2,3,4} {
    	\draw (\x,-.1) -- (\x,.1);
        \draw[dotted] (\x,0) -- (\x,2);
    }
    \fill[color=black] 
    		(C1) circle (\r)
    		(C2) circle (\r)
            (C3) circle (\r)
            (C4) circle (\r)
            (C5) circle (\r);   
    
    \draw[very thick,-latex] (0.3,0.3) -- (0.7,0.7);
    \draw[very thick,-latex] (1.3,0.5) -- (1.7,0.9);
    \draw[very thick,-latex] (2.3,0.3) -- (2.7,0.7);
    \draw[very thick,-latex] (3.3,0.7) -- (3.7,0.3);
    
    \draw (2,1.8)--(2,2.2) (4,1.8)--(4,2.2);
    \draw[<->] (2,2)--(4,2);

    
     \begin{scope}[shift={(6.5,0)},scale = 0.5]\scriptsize
     
     \def\r{0.12};
	
	\draw[thick] (0,3.5) .. controls (4,3) and (5,3) .. (6,0)
    	.. controls (2,0.5) and (1,0.5) .. (0,3.5);
    \draw[-latex] (0,-.5)--(0,4.5); 
    \draw[-latex] (-.5,0)--(7,0);
    \node[left] at (0,4.5) {$TN$};
    \node[below] at (7,0) {$TP$};
    \draw (6,0) +(0,0.25) -- +(0,-.25);
    \draw (0,3.5) +(.25,0) -- +(-.25,0);
    
    \coordinate (C*) at (4.5,2.45);
    \coordinate (l*) at (2.3,4.25);
    \coordinate (Ct) at (2.8,0.47);
    \coordinate (lt) at (-0.2,1.25); 
    
    \fill[color=black] (0,3.5) circle (\r)
    		(6,0) circle (\r)
            (3,1.75) circle (0.08)
            (C*) circle (\r)
            (Ct) circle (\r);
    \node[right] at (l*) {$\bar{\ell}_f^* = \bar{\ell}$};
    \node[below left] at (lt) {$\barbelow{\ell}^*_f = \barbelow{\ell}$};
    
    \draw (C*) +(-2.3,1.7) -- +(2.3,-1.7);
    \draw[-latex] (C*) +(0,0) -- +(0.68,0.92);
    \node[left] at (C*) {$\bar{C}^*$};
    \node[below right] at ($(C*)+(0.34,0.46)$) {$\nabla \phi^*$};
   	
    \draw (Ct) +(-4,1) -- +(3,-0.75);
    \node[below left] at ($(Ct)+(-0.15,0.30)$) {$\barbelow{C}^*$};
     \end{scope}
\end{tikzpicture}
	\caption{\small{(a)} \normalsize  Supporting hyperplanes (with normal vectors) and resulting geometry of $\Ccal$; \small(b) \normalsize Sketch of Algorithm~\ref{alg:linear}; \small(c) \normalsize Maximizer $\Cbar^*$ and minimizer $\barbelow{C}^*$ along with the supporting hyperplanes for LFPMs.}
	\label{fig:lin-fr}
\end{figure*}

ME will require confusion matrices that are achieved by all possible 
classifiers, thus it is necessary to characterize the set $\Ccal$ in a way which is useful for the task.

\bassumption
We assume $g(t)=\Pmbb[\eta(X)\geq t]$ is continuous and 
strictly decreasing for $t \in [0, 1]$.
\label{as:eta}
\eassumption

This is equivalent to standard assumptions \cite{koyejo2014consistent} that the event $\eta(X)=t$ has positive density but zero probability. 
Note that this requires $X$ to have no point mass.

\bprop{(Properties of $\Ccal$ --- Figure~\ref{fig:lin-fr}(a).)}\label{pr:strict-convex}
The set of confusion matrices $\Ccal$ is convex, closed, contained in the rectangle $[0,\zeta]\times [0,1-\zeta]$ (bounded), and $180\degree$ rotationally symmetric around the center-point $(\frac{\zeta}{2}, \frac{1-\zeta}{2})$. Under Assumption~\ref{as:eta}, $(0,1-\zeta)$ and $(\zeta,0)$ are the only vertices of $\Ccal$, and $\Ccal$ is strictly convex. Thus, any supporting hyperplane of $\Ccal$ is tangent at only one point.\footnote{Additional visual intuition about the geometry of C (via an example) is given in Appendix~\ref{appendix:visualization}.}
\eprop

\subsection{LPM Parametrization and Connection with Supporting Hyperplanes of $\Ccal$}
\label{ssec:parametrization}
For an LPM $\phi$ \eqref{eq:linear}, Proposition~\ref{pr:strict-convex} guarantees the existence of a unique Bayes confusion matrix on the boundary $\partial \Ccal$.  
This is because optimum for a linear function over a strictly convex set is unique and lies on the boundary~\cite{boyd2004convex}.
Note that any linear function with the same trade-offs for \TPs and \TN, i.e. same $(m_{11}, m_{00})$, is maximized at the same boundary point regardless of the bias term $m_0$. 
Thus, different LPMs can be generated by varying trade-offs $\mmbf = (m_{11}, m_{00})$ such that $\norm{\mmbf} = 1$ and $m_0 = 0$. The condition $\norm{\mmbf} = 1$ does not affect the learning problem as discussed in Example~\ref{ex:lossbased}. In other words, the performance metric is scale invariant. 
This allows us to represent the family of linear metrics $\varphi_{LPM}$ by a single parameter~$\theta \in [0, 2\pi]$:
\begin{ceqn}
\begin{equation}
	\varphi_{LPM} = \{ \mmbf = (\cos\theta, \sin\theta) : \theta \in [0, 2\pi]\}.
	\label{set:lfpm}
\end{equation}
\end{ceqn}
Given $\mmbf$ (equiv. to $\theta$), we can recover the Bayes classifier using Proposition~\ref{pr:bayeslinear}, and then the Bayes confusion matrix  $\Cbar_\theta$ = $\Cbar_\mmbf = (\oline{TP}_{\mmbf}, \oline{TN}_{\mmbf})$ using \eqref{eq:components}. Under Assumption~\ref{as:eta}, due to strict convexity of $\Ccal$, the Bayes confusion matrix $\Cbar_\mmbf$ is unique; therefore, we have that 
\begin{ceqn}
\begin{align}
\langle \mmbf, C \rangle <  \langle \mmbf, \Cbar_\mmbf \rangle \qquad \forall \; C \in \Ccal, C \neq \Cbar_\mmbf.
\end{align}
\end{ceqn}
Notice the connection between the linear performance metrics and the supporting hyperplanes of the set $\Ccal$ (see Figure~\ref{fig:lin-fr}(a)). Given $\mmbf$, there exists a supporting hyperplane tangent to $\Ccal$ at only $\Cbar_\mmbf$ defined as follows:
\begin{ceqn}
\bequation
\bell_\mmbf \coloneqq m_{11}\cdot tp + m_{00}\cdot tn = m_{11}\oline{TP}_{\mmbf} +  m_{00}\oline{TN}_{\mmbf}.
\label{eq:support}
\eequation
\end{ceqn}

Clearly, if $m_{11}$ and $m_{00}$ are of opposite sign (i.e., $\theta \in (\sfrac{\pi}{2}, \pi)\cup(\sfrac{3\pi}{2}, 2\pi)$), then $\hbar_{\mmbf}$ is the trivial classifier predicting either 1 or 0 everywhere. In other words, if the slope of the hyperplane is positive, then it touches the set $\Ccal$ either at $(\zeta, 0)$ or $(0, 1 - \zeta)$.  
When $m_{11}, m_{00} \neq 0$ with the same sign (i.e., $\theta \in (0, \sfrac{\pi}{2})\cup(\pi, \sfrac{3\pi}{2})$), 
then the Bayes confusion matrix is away from the two vertices. Now, we may split the boundary $\partial\Ccal$ as follows:
\bdefinition 
\label{def:boundary}
The Bayes confusion matrices for LPMs with $m_{11}, m_{00} \geq 0$ $(\theta \in [0, \sfrac{\pi}{2}])$ form the upper boundary, denoted by $\partial\Ccal_+$. The Bayes confusion matrices for LPMs with $m_{11}, m_{00} < 0$ $(\theta \in (\pi, \sfrac{3\pi}{2}))$ form the lower boundary, denoted by $\partial\Ccal_-$. From Proposition ~\ref{pr:bayeslinear}, it follows that the confusion matrices in $\partial\Ccal_+$ and $\partial\Ccal_-$ correspond to the classifiers of the form  $\1[\eta(x)\geq \delta]$ and $\1[\delta \geq \eta(x)]$, respectively, for some $\delta \in [0, 1]$.
\edefinition
\section{Algorithms}
\label{sec:algorithms}

In this section, we propose binary-search type algorithms, which exploit the geometry of the set $\Ccal$ (Section \ref{sec:confusion}) to find the maximizer / minimizer and the associated supporting hyperplanes for any quasiconcave / quasiconvex metrics. These algorithms are then used to elicit LPMs and LFPMs, both of which belong to both quasiconcave and quasiconvex function families.

We allow \emph{noisy} oracles; however, for simplicity, 
we will first discuss algorithms and elicitation with no-noise, and then show that they are robust to the noisy feedback (Section~\ref{sec:noise}). 
Moreover, as one typically prefers metrics which reward correct classification, we first discuss metrics that are monotonically increasing in both $\TP$ and $\TN$. 
The monotonically decreasing case is discussed in Appendix~\ref{appendix:decreasing} as a natural extension. 

The following lemma for any quasiconcave and quasiconvex metrics forms the basis of our proposed algorithms. 
\blemma
Let 
$\rho^+:[0,1]\to \partial\mathcal C_+$, $\rho^-:[0,1]\to \partial\mathcal C_-$
be continuous, bijective, parametrizations of the upper and lower boundary,
respectively. Let $\phi:\mathcal C\to \mathbb R$ be a quasiconcave function,
and $\psi:\mathcal C\to \mathbb R$ be a quasiconvex function, 
which are monotone increasing in both $TP$ and $TN$.
Then the composition $\phi\circ \rho^+: [0,1]\to\mathbb R$ is
quasiconcave (and therefore unimodal) on the interval $[0, 1]$, and $\psi\circ\rho^-:[0,1]\to\mathbb R$ is quasiconvex (and therefore unimodal) on the interval $[0,1]$. \label{lem:quasiconcave}
\elemma

The unimodality of quasiconcave (quasiconvex) metrics on the upper (lower) boundary of the set $\Ccal$ 
along with the one-dimensional parametrization of $\mmbf$ using $\theta \in [0, 2\pi]$ (Section \ref{sec:confusion}) allows us to devise binary-search-type methods to find the maximizer $\Cbar$, the minimizer $\barbelow{C}$, and the first order approximation of $\phi$ at these points, i.e., the supporting hyperplanes at $\Cbar$ and $\barbelow{C}$. 
\balgorithm[t]
\caption{Quasiconcave Metric Maximization}
\label{alg:linear}
\balgorithmic[1]
\STATE \textbf{Input:} $\epsilon > 0$ and oracle $\Omega$. 
\STATE \textbf{Initialize:} $\theta_a = 0$, $\theta_b = \frac{\pi}{2}$.
\WHILE {$\abs{\theta_b - \theta_a} > \epsilon$}
\STATE Set $\theta_c = \frac{3\theta_a + \theta_b}{4}$, $\theta_d = \frac{\theta_a + \theta_b}{2}$, and $\theta_e = \frac{\theta_a + 3\theta_b}{4}$. Set corresponding slopes ($\mmbf$'s) using \eqref{set:lfpm}. \\
\STATE Obtain $\hbar_{\theta_a}$,$\hbar_{\theta_c}$,$\hbar_{\theta_d}$, $\hbar_{\theta_e}, \hbar_{\theta_b}$ using Proposition \ref{pr:bayeslinear}. Compute $\Cbar_{\theta_a}$,$\Cbar_{\theta_c}$,$\Cbar_{\theta_d}$,$\Cbar_{\theta_e}, \Cbar_{\theta_b}$ using~\eqref{eq:components}.
\STATE Query $\Omega(\Cbar_{\theta_c}, \Cbar_{\theta_a}), \Omega(\Cbar_{\theta_d}, \Cbar_{\theta_c}), \Omega(\Cbar_{\theta_e}, \Cbar_{\theta_d}),$ and $\Omega(\Cbar_{\theta_b}, \Cbar_{\theta_e})$.
\STATE If $\Cbar_\theta \succ \Cbar_{\theta'}\prec \Cbar_{\theta''}$ for consecutive $\theta<\theta'<\theta''$, assume the default order $\Cbar_\theta \prec \Cbar_{\theta'}\prec \Cbar_{\theta''}.$
\STATE {\bf if }($\Cbar_{\theta_a}^* \succ \Cbar_{\theta_c}^*$) Set $\theta_b = \theta_d$.
\STATE {\bf elseif }($\Cbar_{\theta_a}^* \prec \Cbar_{\theta_c}^* \succ \Cbar_{\theta_d}^*$) Set $\theta_b = \theta_d$.
\STATE {\bf elseif }($\Cbar_{\theta_c}^* \prec \Cbar_{\theta_d}^* \succ \Cbar_{\theta_e}^*$) Set $\theta_a = \theta_c$,  $\theta_b = \theta_e$.
\STATE {\bf elseif }($\Cbar_{\theta_d}^* \prec \Cbar_{\theta_e}^* \succ \Cbar_{\theta_b}^*$) Set $\theta_a = \theta_d$.
\STATE {\bf else } 
 Set $\theta_a = \theta_d$.
\ENDWHILE
\STATE\textbf{Output:} $\bmmbf, \Cbar,$ and $\bell$, where $\bmmbf = \mmbf_d$ ($\theta_d$), $\Cbar = {\Cbar}_{\theta_d},$ and $\bell := \langle \bmmbf, (tp, tn) \rangle = \langle \bmmbf, {\Cbar} \rangle$.
\ealgorithmic
\ealgorithm

\textbf{Algorithm \ref{alg:linear}.} \emph{Maximizing quasiconcave metrics and finding supporting hyperplanes at the optimum:} Since $\phi$ is monotonically increasing in both \TPs and \TN, and $\mathcal C$ is convex, the maximizer must be on the upper boundary. 
Hence, we start with the interval $[\theta_a = 0, \theta_b = \frac{\pi}{2}]$ (Definition~\ref{def:boundary}). We divide it into four equal parts and set slopes using \eqref{set:lfpm} in line 4 (see  Figure~\ref{fig:lin-fr}(b) for visual intuition). 
Then, we compute the Bayes classifiers using Proposition \ref{pr:bayeslinear} and the associated Bayes confusion matrices in line 5. We pose four pairwise queries to the oracle in line 6. 
Line 7 gives the default direction to binary search in case of out-of-order responses.\footnote{Due to finite samples, $\Ccal$'s boundary may have staircase-type bumps in practice. This may lead to out-of-order responses, even when the metric is unimodal \emph{w.r.t.} $\theta$.} In lines 8-12, we shrink the search interval by half based on oracle responses. 
We stop when the search interval becomes smaller than a given $\epsilon > 0$ (tolerance).
Lastly, we output the slope $\bmmbf$, the Bayes confusion matrix $\Cbar$, and the supporting hyperplane $\bell$ at that point.

\textbf{Algorithm \hypertarget{alg:quasiconvex}{\hyperlink{alg:quasiconvex}{2}}}\textbf{.} \emph{Minimizing quasiconvex metrics and finding supporting hyperplane at the optimum:} The same algorithm can be used for quasiconvex minimization with only two changes. First, we start with $\theta \in [\pi,\frac 32\pi]$, because the optimum will lie on the lower boundary $\partial\mathcal C_-$. Second, we check for $C\prec C'$ whenever Algorithm~\ref{alg:linear} checks for $C\succ C'$, and vice versa. 
Here, we output the counterparts, i.e., slope $\tmmbf$,  inverse Bayes Confusion matrix $\barbelow{C}$, and supporting hyperplane $\tell$.  

\section{METRIC ELICITATION}
\label{sec:me}
\begin{figure}[t] \centering \fbox{ \parbox{0.95\columnwidth}{ {
\textbf{LPM Elicitation} (True metric $\sphi = \smmbf$)
\benumerate
\item Run Algorithm~\ref{alg:linear} to get $\Cbar^*$ and a hyperplane $\bell$.
\item Set the elicited metric to be the slope of $\bell$.
\eenumerate

\textbf{LFPM Elicitation} (True metric $\sphi$)
\benumerate
\item Run Algorithm~\ref{alg:linear} to get $\Cbar^*$, a hyperplane $\bell$, and SoE~\eqref{eq:lin-fr-equi}.
\item Run Algorithm~\hyperlink{alg:quasiconvex}{2} to get $\barbelow{C}^*$, a hyperplane $\tell$, and SoE~\eqref{eq:lin-fr-equi-lower}.
\item Run the oracle-query independent Algorithm~\ref{alg:grid-search} to get the elicited metric, which satisfies both the SoEs.
\eenumerate
}
}
}
\caption{LPM and LFPM elicitation procedures.}
\label{fig:me}
\end{figure}

In this section, we discuss how Algorithms~\ref{alg:linear}, \hyperlink{alg:quasiconvex}{2},  and~\ref{alg:grid-search} (described later) are used as subroutines to elicit LPMs and LFPMs. See Figure~\ref{fig:me} for a brief summary.

\subsection{Eliciting LPMs}
\label{ssec:elicit_linear}

Suppose that the oracle's metric is $ \varphi_{LPM} \ni \sphi = \smmbf$, where, WLOG, $\norm{\smmbf} =1$ and $m_0^* = 0$ (Section \ref{sec:confusion}). 
Application of Algorithm~\ref{alg:linear} to the oracle, who responds according to $\smmbf$, returns the maximizer and supporting hyperplane at that point. Since the true performance metric is linear, we take the elicited metric, $\hmmbf$, to be the slope of the resulting supporting hyperplane.
\subsection{Eliciting LFPMs}
\label{ssec:elicit_linearfrac}

An LFPM is given by \eqref{linear-fractional}, where $p_{11}, p_{00}, q_{11}$, and $q_{00}$ are not simultaneously zero. Also, it is bounded over $\Ccal$. As scaling and shifting does not change the linear-fractional form, \emph{WLOG}, we may take $\phi(C) \in [0, 1] \, \forall  C \in \Ccal$ with positive numerator and denominator.

\bassumption
Let $\phi \in \varphi_{LFPM}$ \eqref{linear-fractional}. We assume that $p_{11}, p_{00} \geq 0$, $p_{11} \geq q_{11}$, $p_{00} \geq q_{00}$, $p_0 = 0$, $q_0 = (p_{11} - q_{11})\zeta + (p_{00} - q_{00})(1 - \zeta)$, and $p_{11} + p_{00} = 1$.
\label{assump:sufficient}
\eassumption

\bprop
The conditions in Assumption \ref{assump:sufficient} are sufficient for $\phi \in \varphi_{LFPM}$ to be bounded in $[0,1]$ and simultaneously monotonically increasing in TP and TN.
\label{prop:sufficient}
\eprop

The conditions in Assumption~\ref{assump:sufficient} 
are reasonable as we want to elicit any unknown bounded, monotonically increasing LFPM. To no surprise, examples outlined in \eqref{ex:lf-examples} and Koyejo et al. \cite{koyejo2014consistent} satisfy these conditions. We first provide intuition for eliciting LFPMs (Figure~\ref{fig:me}). We obtain two hyperplanes: one at the maximizer on the upper boundary, and other at the minimizer on the lower boundary. This results in two nonlinear systems of equations (SoEs) having only one degree of freedom, but they are satisfied by the true unknown metric. 
Thus, the elicited metric is one where solutions to the two systems match pointwise on the confusion matrices. Formally, suppose that the oracle's metric is:

\begin{ceqn}
\begin{align}
\sphi(C) = \frac{\spone TP +  \spzero TN}{\sqone TP +  \sqzero TN +  \sqnot}. \nonumber
\end{align}
\end{ceqn}

Let $\btau^*$ and $\barbelow{\tau}^*$ be the maximum and minimum value of $\sphi$ over $\Ccal$, respectively, i.e., 
$\underline{\tau}^*\leq \phi^*(C)\leq \overline{\tau}^*\; \forall \; C\in \Ccal$. 
Under Assumption  \ref{as:eta}, we have a hyperplane 
\begin{ceqn}
\begin{align}
\bell_f^* := (\spone - \btau^*\sqone)tp +  (\spone - \btau^*\sqone)tn = \btau^*\sqnot \nonumber
\label{eq:lf-support}
\end{align}
\end{ceqn}
touching the set $\Ccal$ only at $(\oline{TP}^*, \oline{TN}^*)$ on the upper boundary $\partial \Ccal_+$.  
Similarly, we have a hyperplane
\begin{ceqn}
\bequation
{\tell}^*_f := (\spone - \ttau^* \sqone)tp +  (\spzero - \ttau^* \sqzero)tn = \ttau^* \sqnot, \nonumber
\label{eq:lf-lower-support}
\eequation
\end{ceqn}
which touches the set $\Ccal$ only at $(\barbelow{TP}^*, \barbelow{TN}^*)$ on the lower boundary $\partial \Ccal_-$. 
To help with intuition, see Figure \ref{fig:lin-fr}(c). 
Since LFPM is quasiconcave, Algorithm \ref{alg:linear} returns a hyperplane
${\bell := \bmone tp + \bmzero tn = \oline{C}_0}$,
where $\oline{C}_0 = \bmone \oline{TP}^* + \bmzero \oline{TN}^*$. This is equivalent to $\bell_f^*$ up to a constant multiple; therefore, the true metric is the solution to the following non-linear SoE:
\begin{ceqn}
\begin{align}
\spone - \btau^*\sqone = \alpha \bmone,  \spzero - \btau^*\sqzero = \alpha \bmzero, \nonumber 
  \btau^*\sqnot = \alpha \oline{C}_0, \nonumber 
\end{align}
\end{ceqn}
where $\alpha \geq 0$, because LHS and $\oline{m}$'s are non-negative. Additionally, we ignore the case when $\alpha = 0$, since this would imply a constant $\phi$. 
Next, we may divide the above equations by $\alpha > 0$ on both sides so that all the coefficients $\oline{p}^*$'s and $\oline{q}^*$'s are factored by $\alpha$. This does not change $\sphi$; thus, the SoE becomes:

\begin{align}
\ppone - \btau^*\pqone =  \bmone,  \ppzero - \btau^*\pqzero =  \bmzero, 
 \btau^*\pqnot =  \oline{C}_0.
\label{eq:lin-fr-equi}
\end{align}

Notice that none of the conditions in Assumption \ref{assump:sufficient} are changed except $\ppone + \ppzero = 1$. However, we may still use this condition to learn a constant $\alpha$ times the true metric, which does not harm the elicitation problem. 

As LFPM is also quasiconvex, Algorithm \hyperlink{alg:quasiconvex}{2} outputs a hyperplane 
${{\tell} := {\barbelow{m}}_{11} tp + {\barbelow{m}}_{00} tn = {\barbelow{C}}_0},$
where ${\barbelow{C}}_0 = {\barbelow{m}}_{11} {\barbelow{TP}}^* + {\barbelow{m}}_{00} {\barbelow{TN}}^*$. This is equivalent to ${\tell}^*_f$ up to a constant multiple; thus, the true metric is also the solution of the following SoE:
\begin{ceqn}
\begin{align}
\spone - {\ttau}^* \sqone = \gamma {\barbelow{m}}_{11}, 
\spzero - {\ttau}^* \sqzero = \gamma {\barbelow{m}}_{00},
 {\ttau}^* \sqnot = \gamma \barbelow{C}_0, \nonumber
\end{align}
\end{ceqn}
where $\gamma \leq 0$ since LHS is positive, but $\barbelow{m}$'s are negative. Again, we may assume $\gamma < 0$. By dividing the above equations by $-\gamma$ on both sides, all the coefficients ${p}^*$'s and ${q}^*$'s are factored by $-\gamma$. This does not change $\sphi$; thus, the system of equations becomes the following:
\begin{ceqn}
\begin{align}
\pppone - {\ttau}^* \ppqone =  {\barbelow{m}}_{11}, 
\pppzero - {\ttau}^* \ppqzero =  {\barbelow{m}}_{00}, 
  {\ttau}^* \ppqnot =  \barbelow{C}_0.
\label{eq:lin-fr-equi-lower}
\end{align}
\end{ceqn}

\bprop
Under Assumption~\ref{assump:sufficient}, knowing $p_{11}'$ solves the system of equations~\eqref{eq:lin-fr-equi} as follows:
\begin{ceqn}
\begin{align}
	p_{00}' &= 1 - p_{11}', \, \nonumber q_0' = \oline{C}_0\frac{P'}{Q'}, \nonumber \\
	q_{11}' &= (p_{11}' - \bmone)\frac{P'}{Q'}, \, q_{00}' = (p_{00}' - \bmzero)\frac{P'}{Q'}, 
\end{align}
\end{ceqn}
where $P' = p_{11}'\zeta + p_{00}'(1 - \zeta)$ and $Q' = P' + \oline{C}_0 - \bmone\zeta -  \bmzero(1 - \zeta)$. Thus, it elicits the LFPM.
\label{pr:solvesystem}
\eprop

\addtocounter{algorithm}{1}
\balgorithm[t]
\caption{Grid Search for Best Ratio}
\label{alg:grid-search}
\balgorithmic[1]
\STATE \textbf{Input:} $k, \Delta$. 
\STATE \textbf{Initialize:} $\sigma_{opt} = \infty, p_{11, opt}' = 0$.
\STATE Generate $C_1,...,C_k$ on $\partial C_+$ and $\partial C_-$ (Section \ref{sec:confusion}).
\FOR {($\ppone = 0$; $\ppone \leq 1$; $\ppone = \ppone + \Delta$)}
\STATE Compute $\pphi$, $\ppphi$ using Proposition \ref{pr:solvesystem}. Compute array $r = [\frac{\pphi(C_1)}{\ppphi(C_1)},...,\frac{\pphi(C_k)}{\ppphi(C_k)}]$. Set $\sigma = \text{std}(r).$\\
\STATE {\bf if }($\sigma < \sigma_{opt}$) Set $\sigma_{opt} = \sigma$ and $p_{11, opt}' = \ppone$.
\ENDFOR
\STATE\textbf{Output:} $p_{11, opt}'$.
\ealgorithmic
\ealgorithm
Now assume we know $p_{11}'$. Using Proposition~\ref{pr:solvesystem}, we may solve the system~\eqref{eq:lin-fr-equi} and obtain a metric, say $\phi'$. System~\eqref{eq:lin-fr-equi-lower} can be solved analogously, provided we know $p_{11}''$, to get a metric, say $\phi''$. Notice that when $\sfrac{\spone}{\spzero}=\sfrac{\ppone}{\ppzero}=\sfrac{\pppone}{\pppzero}$, then $\sphi (C) = \phi'(C)/\alpha = -\phi''(C)/\gamma$. This means that when the  true ratios of $p$'s are known, then $\phi'$, $\phi''$ are constant multiples of each other. So, to know the true $\ppone$ (or, $\pppone$) is to search the grid $[0,1]$ and select the one where the ratios of $\phi'$ and $\phi''$ are constant on a number of confusion matrices. Since we can generate many confusion matrices on $\partial \Ccal_+$ and $\partial \Ccal_-$ (vary $\delta$ in Definition~\ref{def:boundary}), we can estimate the ratio $p_{11}'$ to $p_{00}'$ using grid search based Algorithm~\ref{alg:grid-search}. We may then use Proposition~\ref{pr:solvesystem} for the output of Algorithm~\ref{alg:grid-search} and set the elicited metric $\hphi = \pphi$. 
Note that Algorithm~\ref{alg:grid-search} is independent of oracle queries and easy to implement, thus it is suitable for the purpose.
\section{Guarantees}
\label{sec:noise}

In this section, we discuss guarantees for the elicitation procedures (Section~\ref{sec:me}) in the presence of (a) confusion matrices' estimation noise from finite samples and (b) oracle feedback noise with the following notion.
\bdefinition
Oracle Feedback Noise $(\epsilon_\Omega\geq0)$: 
The oracle may provide wrong answers whenever $|\phi(C)-\phi(C')|<\epsilon_\Omega$. Otherwise, it provides correct answers. 
\edefinition

Simply put, if the confusion matrices are close as measured by $\phi$, then the oracle responses can be wrong. Moving forward to the guarantees, we make two assumptions which hold in most common settings.

\begin{assumption}\label{as:sup-norm-convergence}
	Let $\{\hat \eta_i(x)\}_{i=1}^{n}$ be a sequence of estimates of $\eta(x)$ depending on the sample size. 
	We assume that $\Vert\eta - \hat \eta_i\Vert_\infty \stackrel{P}{\to} 0$.
\end{assumption}
\begin{assumption}\label{as:low-weight-around-opt}
	For quasiconcave $\phi$, recall that the Bayes classifier 
	is of the form $h = \1[\eta(x)\geq \delta]$. Let $\oline{\delta}$ be the threshold that maximizes $\phi$. We assume that the probability that $\eta(X)$ lies near $\oline{\delta}$ is bounded from below and above. Formally, $k_0 \nu \leq \Pmbb\left[(\oline{\delta}-\eta(X))\in[0,\nu]\right], $
    $\Pmbb\left[(\eta(X)-\oline{\delta})\in[0,\nu]\right]\leq k_1 \nu$
    for any $0<\nu\leq\frac{2}{k_0}\sqrt{k_1\epsilon_\Omega}$ and some $k_1 \geq k_0 > 0$.
\end{assumption}
Assumption \ref{as:sup-norm-convergence} is arguably natural, as most estimation is parametric, where the function classes are sufficiently well behaved. Assumption \ref{as:low-weight-around-opt} ensures 
that near the optimal threshold $\oline{\delta}$, the values of $\eta(X)$ have bounded density. 
In other words, when $X$ has no point mass, the slope of $\eta(X)$ where it attains the optimal threshold $\oline{\delta}$ is neither vertical nor horizontal. We start with guarantees for the algorithms in their respective tasks.

\begin{theorem}\label{thm:quasi}
Given $\epsilon,\epsilon_\Omega \geq 0$ and a 1-Lipschitz metric $\phi$ that is monotonically increasing in TP, TN. If it is quasiconcave (quasiconvex) then Algorithm \ref{alg:linear} (Algorithm~\hyperlink{alg:quasiconvex}{2}) finds an approximate maximizer $\Cbar$ (minimizer $\barbelow{C}$). Furthemore, $(i)$ the algorithm returns the supporting hyperplane at that point, $(ii)$ 
the value of $\phi$ at that point is within $O(\sqrt{\epsilon_\Omega} +  \epsilon)$ of the optimum, and $(iii)$ the number of queries is $O(\log\frac1\epsilon)$.
\end{theorem}
\blemma
\label{lem:lower-bound}
Under our model, no algorithm can find the maximizer (minimizer) in fewer than~$O(\log\frac1\epsilon)$ queries.
\elemma
Theorem~\ref{thm:quasi} and Lemma~\ref{lem:lower-bound}, guarantee that Algorithm~\ref{alg:linear} (Algorithm \hyperlink{alg:quasiconvex}{2}), for a quasiconcave (quasiconvex) metric, finds a confusion matrix and a hypeplane which is close to the true maximizer (minimizer) and its associated supporting hyperplane, using just the optimal number of queries. Further, since binary search always tends towards the optimal whenever responses are correct, the algorithms necessarily terminate within a confidence interval of the true maximizer. Thus, we can take $\epsilon$ sufficiently small so that the only error that arises is due to the feedback noise $\epsilon_\Omega$. Now, we present our main result which guarantees effective LPM elicitation. 
Guarantees in LFPM elicitation follow naturally as discussed in the proof of Theorem~\ref{thm:linear}  (Appendix~\ref{appendix:proofs}).

\btheorem\label{thm:linear}
Let $\varphi_{LPM} \ni \sphi = \smmbf$ be the true performance metric. Under Assumption~\ref{as:low-weight-around-opt}, given $\epsilon > 0$, LPM elicitation (Section~\ref{ssec:elicit_linear}) outputs a 
performance metric $\hphi = \hmmbf$, such that $\norm{\smmbf - \hmmbf}_\infty \leq \sqrt{2}\epsilon + \frac 2{k_0}\sqrt{2k_1\epsilon_\Omega}$.
\etheorem

So far, we assumed access to the confusion matrices. However, in practice, we need to estimate them using samples $\{(x_i, y_i)\}_{i=1}^{n}$. 
We now discuss robustness of the algorithms working with samples. 
 Recall that, as a standard consequence of Chernoff-type bounds~\cite{boucheron2013concentration}, sample estimates of true-positive and true-negative are consistent estimators. 
Therefore, with high probability, we can estimate the confusion matrix within any desired tolerance, provided we have sufficient samples.  This implies that we can also estimate the $\phi$ values within any tolerance since LPM and and LFPM are 1-Lipschitz due to \eqref{set:lfpm} and  Assumption \ref{assump:sufficient}, respectively. 
Thus, with high probability, the elicitation procedures gather correct oracle's preferences within feedback noise $\epsilon_\Omega$. 
Further, we may prove the following lemma which allow us to control the 
 error in optimal classifiers from using the  
 estimated $\hat\eta(x)$ rather than the true $\eta(x)$.
 
\begin{lemma}\label{lem:sample-Cs-optimize-well}
Let $h_{\theta}$ and $\hat h_{\theta}$ be two classifiers estimated using $\eta$ and $\hat\eta$, respectively. 
Further, let ${\oline{\theta}}$ be such that $h_{{\oline{\theta}}} = \argmax_{\theta}\phi(h_{\theta})$. Then
	${\Vert C(\hat h_{{\oline{\theta}}}) - C(h_{{\oline{\theta}}}) \Vert_\infty=O( \Vert {\hat\eta}_n-\eta\Vert_\infty})$.
\end{lemma}

The errors due to using $\hat\eta$, instead of true $\eta$ may propel in the results discussed earlier, however, only in the bounded sense. This shows that our elicitation approach is robust to feedback and finite sample noise.
\section{Experiments}
\label{experiments}

In this section, we empirically validate the theory and investigate the sensitivity due to sample estimates.

\subsection{Synthetic Data Experiments}
\label{ssec:theoryexp}

We assume a joint probability for $\Xcal = [-1,1]$ and $\Ycal = \{0, 1\}$ given by $f_X = \Umbb[-1,1]$ and $\eta(x) = \frac{1}{1 + e^{ax}}$, where $\Umbb[-1,1]$ is the uniform distribution on $[-1, 1]$, and $a$ is a parameter controlling the degree of noise in the labels. We fix $a = 5$ in our experiments. In the LPM elicitation case, we define a true metric $\sphi$ by $\smmbf = (\smone, \smzero)$. This defines the query outputs in line 6 of Algorithm \ref{alg:linear}. Then we run Algorithm \ref{alg:linear} to check whether or not we get the same metric. The results for both monotonically increasing and monotonically decreasing LPM are shown in Table~\ref{tab:app:LPMtheory}. We achieve the true metric even for very tight tolerance $\epsilon=0.02$ radians.

\begin{table}
	\caption{Empirical Validation for LPM elicitation at tolerance $\epsilon = 0.02$ radians. $\sphi$ and $\hphi$ denote the true and the elicited metric, respectively.}
	\label{tab:app:LPMtheory}
	\begin{center}
		\begin{small}
				\begin{tabular}{|c|c|c|c|}
					\hline
					  $\sphi = \smmbf$ & $\hphi = \hmmbf$ & $\sphi = \smmbf$ & $\hphi = \hmmbf$ \\ \hline 
					(0.98,0.17) & (0.99,0.17) & (-0.94,-0.34) & (-0.94,-0.34) \\
					(0.87,0.50) & (0.87,0.50) &(-0.77,-0.64)& (-0.77,-0.64)  \\
					(0.64,0.77) & (0.64,0.77) & (-0.50,-0.87) & (-0.50,-0.87)  \\
					(0.34,0.94) &(0.34,0.94) &(-0.17,-0.98) & (-0.17,-0.99 ) \\
					\hline
				\end{tabular}
		\end{small}
	\end{center}
\end{table}

\begin{table*}
	\caption{LFPM Elicitation for synthetic distribution (Section \ref{ssec:theoryexp}) and Magic (\textsc{M}) dataset  (Section \ref{ssec:app:realexp}) with $\epsilon = 0.05$ radians. $(\spone, \spzero), (\sqone, \sqzero, \sqnot)$ denote the true LFPM. $(\hpone, \hpzero), (\hqone, \hqzero, \hqnot)$ denote the elicited LFPM. $\alpha$ and $\sigma$ denote the mean and the standard deviation in the ratio of the elicited to the true metric (evaluated on the confusion matrices in $\partial\Ccal_+$ used in Algorithm~\ref{alg:grid-search}), respectively. We empirically verify that the elicited metric is constant multiple ($\alpha$) of the true metric.}
	\label{tab:app:LFPMtheoryreal}

	\begin{center}
		\begin{small}
            \resizebox{\textwidth}{!}{%
			\begin{tabular}{|c|c|c|c|c|c|c|}
				\hline
                
				True Metric & \multicolumn{3}{|c|}{Results on Synthetic Distribution (Section \ref{ssec:theoryexp})} & \multicolumn{3}{|c|}{Results on Real World Dataset \textsc{M}  (Section \ref{ssec:app:realexp})} \\ \hline
				$(\spone, \spzero), (\sqone, \sqzero, \sqnot)$ & $(\hpone, \hpzero), (\hqone, \hqzero, \hqnot)$ & $\alpha$ & $\sigma$ & $(\hpone, \hpzero), (\hqone, \hqzero, \hqnot)$ & $\alpha$ & $\sigma$ \\
				\hline
				(1.00,0.00),(0.50,-0.50,0.50) & (1.00,0.00),(0.25,-0.75,0.75) & 0.92 & 0.03  & (1.00,0.00),(0.25,-0.75,0.75) & 0.90 & 0.06 \\
				(1.0,0.0),(0.8,-0.8,0.5) & (1.0,0.0),(0.73,-1.09,0.68) & 0.94 & 0.02& (1.0,0.0),(0.72,-1.13, 0.57) & 1.06 & 0.05 \\
				(0.8,0.2),(0.3,0.1,0.3)  & (0.86,0.14),(-0.13,-0.07, 0.60) & 0.90 & 0.06 & (0.23,0.77),(-0.87,0.66,0.76) & 0.84  & 0.09\\
				(0.60,0.40),(0.40,0.20,0.20) & (0.67,0.33),(-0.07,-0.44,76) & 0.82 & 0.05 & (0.16,0.84),(-0.89,0.25,0.89) & 0.65 & 0.05\\
				(0.40,0.60),(-0.10,-0.20,0.65) & (0.36,0.64),(-0.21,-0.25,0.73) & 0.97 & 0.01 & (0.08,0.92),(-0.75,0.12,0.82) & 0.79 & 0.08\\
				(0.20,0.80),(-0.40,-0.20,0.80) & (0.12, 0.88),(-0.43, 0.002, 0.71) & 1.02 & 0.006 & (0.19,0.81),(-0.38,-0.13,0.70) & 1.02 & 0.004  \\
				\hline
			\end{tabular}}
		\end{small}
	\end{center}
\end{table*}

Next, we elicit LFPM. We define a true metric $\sphi$ by $\{(\spone, \spzero), (\sqone, \sqzero, \sqnot)\}$.  Then, we run Algorithm \ref{alg:linear} with $\epsilon=0.05$ to find the hyperplane $\bell$ and maximizer on $\partial C_+$, Algorithm~\hyperlink{alg:quasiconvex}{2} with $\epsilon=0.05$ to find the hyperplane $\tell$ and minimizer on $\partial C_-$, and Algorithm \ref{alg:grid-search} with $n = 2000$ (1000 confusion matrices on both $\partial\Ccal_+$ and $\partial\Ccal_-$ obtained by varying parameter $\theta$ uniformly in $[0, \pi/2]$ and $[\pi, 3\pi/2]$) and $\Delta = 0.01$. This gives us the elicited metric $\hphi$, which we represent by $\{(\hpone, \hpzero), (\hqone, \hqzero, \hqnot)\}$.
In Table \ref{tab:app:LFPMtheoryreal}, we present the elicitation results for LFPMs (column 2). We also present the mean ($\alpha$) and the standard deviation ($\sigma$) of the ratio of the elicited metric $\hphi$ to the true metric $\phi$ over the set of confusion matrices (column 3 and 4 of Table \ref{tab:app:LFPMtheoryreal}). Furthermore, if we know the true ratio of $\sfrac{\spone}{\spzero}$, then we can elicit the LFPM up to a constant by only using Algorithm $\ref{alg:linear}$ resulting in better estimate of the true metric, because we avoid errors due to Algorithms~\hyperlink{alg:quasiconvex}{2} and~\ref{alg:grid-search}. Line 1 and line 2 of Table \ref{tab:app:LFPMtheoryreal} represent $F_1$ measure and $F_\frac{1}{2}$ measure, respectively. In both the cases, we assume the knowledge of $p_{11}^*=1$. Line 3 to line 6 correspond to some arbitrarily  chosen linear fractional metrics to show the efficacy of the proposed method. For a better judgment, we show function evaluations of the true metric and the elicited metric on selected pairs of $(TP, TN) \in \partial\Ccal_+$ (used for Algorithm \ref{alg:grid-search}) in Figure \ref{fig:app:lfpm-theory}. The true and the elicited metric are plotted together after sorting values based on slope parameter $\theta$. It is clear that the elicited metric is a constant multiple of the true metric. The vertical solid line in red and dashed line in black corresponds to the \emph{argmax} of the true and the elicited metric, respectively. In Figure \ref{fig:app:lfpm-theory}, we see that the \emph{argmax} of the true and the elicited metrics coincides, thus validating Theorem~\ref{thm:quasi}.

\begin{figure*}[t]
	\centering
	\subfigure[Table \ref{tab:app:LFPMtheoryreal}, Line 1, Column 2]{
		{\includegraphics[width=5cm]{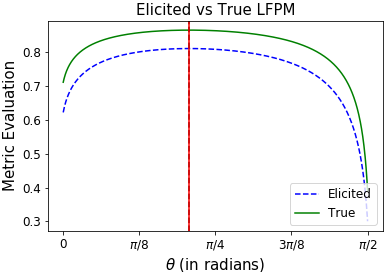}}
		\label{fig:app:lfpm_1}
	}
	\subfigure[Table \ref{tab:app:LFPMtheoryreal}, Line 2, Column 2]{
		{\includegraphics[width=5cm]{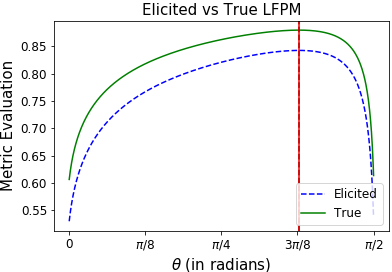}}
		\label{fig:app:lfpm_2}
	}
	\subfigure[Table \ref{tab:app:LFPMtheoryreal}, Line 3, Column 2]{
		{\includegraphics[width=5cm]{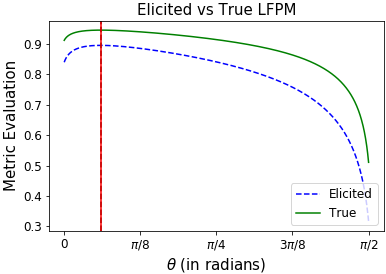}}
		\label{fig:app:lfpm_3}
	}
	\subfigure[Table \ref{tab:app:LFPMtheoryreal}, Line 4, Column 2]{
		{\includegraphics[width=5cm]{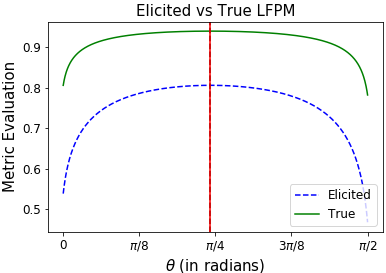}}
		\label{fig:app:lfpm_4}
	}
	\subfigure[Table \ref{tab:app:LFPMtheoryreal}, Line 5, Column 2]{
		{\includegraphics[width=5cm]{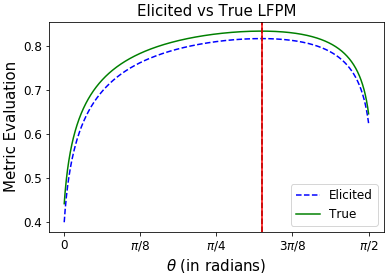}}
		\label{fig:app:lfpm_5}
	}
	\subfigure[Table \ref{tab:app:LFPMtheoryreal}, Line 6, Column 2]{
		{\includegraphics[width=5cm]{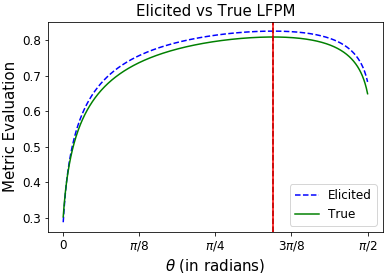}}
		\label{fig:app:lfpm_6}
	}
	\caption{True and elicited LFPMs for synthetic distribution from Table \ref{tab:app:LFPMtheoryreal}. The solid green curve and the dashed blue curve are the true and the elicited metric, respectively. The solid red and the dashed black vertical lines represent the maximizer of the true metric and the elicited metric, respectively. We see that the elicited LFPMs are constant multiple of the true metrics with the same maximizer (solid red and dashed black vertical lines overlap).}
	\label{fig:app:lfpm-theory}
\end{figure*}

\subsection{Real-World Data Experiments}
\label{ssec:app:realexp}

\begin{figure*}[t]
	\centering 
	\subfigure[Table \ref{tab:app:LFPMtheoryreal}, Line 1, Column 5]{
		{\includegraphics[width=5cm]{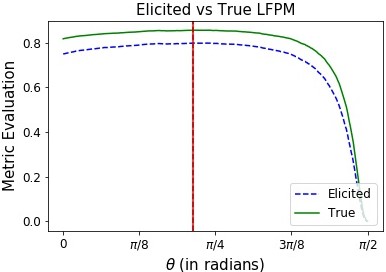}}
    \label{fig:app:lfpm_1_magic}
    }
	\subfigure[Table \ref{tab:app:LFPMtheoryreal}, Line 2, Column 5]{
		{\includegraphics[width=5cm]{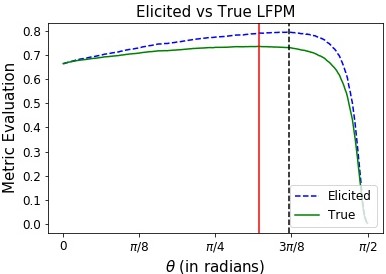}}
		\label{fig:app:lfpm_2_magic}
	}
	\subfigure[Table \ref{tab:app:LFPMtheoryreal}, Line 3, Column 5]{
		{\includegraphics[width=5cm]{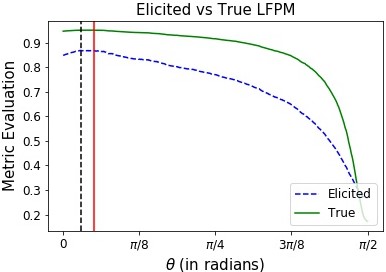}}
		\label{fig:app:lfpm_3_magic}
	}
	\subfigure[Table \ref{tab:app:LFPMtheoryreal}, Line 4, Column 5]{
		{\includegraphics[width=5cm]{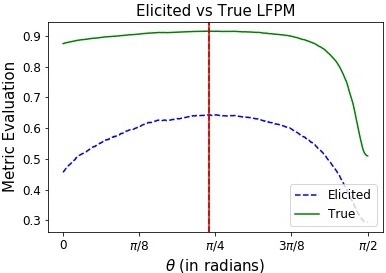}}
		\label{fig:app:lfpm_4_magic}
	}
	\subfigure[Table \ref{tab:app:LFPMtheoryreal}, Line 5, Column 5]{
		{\includegraphics[width=5cm]{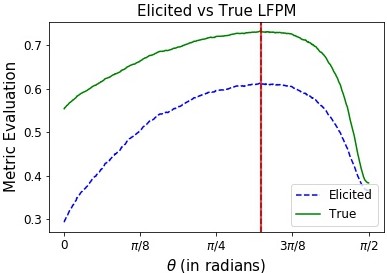}}
		\label{fig:app:lfpm_5_magic}
	}
	\subfigure[Table \ref{tab:app:LFPMtheoryreal}, Line 6, Column 5]{
		{\includegraphics[width=5cm]{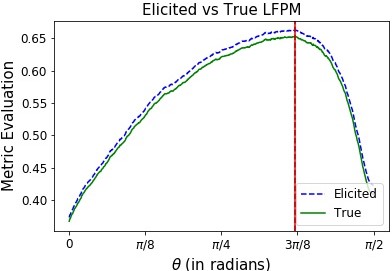}}
		\label{fig:app:lfpm_6_magic}
	}
	\caption{True and elicited LFPMs for dataset \textsc{M} from Table \ref{tab:app:LFPMtheoryreal}. The solid green curve and the dashed blue curve are the true and the elicited metric, respectively. The solid red and the dashed black vertical lines represent the maximizer of the true metric and the elicited metric, respectively. We see that the elicited LFPMs are constant multiple of the true metrics with almost the same maximizer (solid red and dashed black vertical lines overlap except for two cases).}
	\label{fig:app:lfpm-real}
\end{figure*}

In real-world datasets, we do not know $\eta(x)$ and only have finite samples. 
As a result of these two road blocks, the feasible space $\Ccal$ is not as well behaved as shown in Figure \ref{fig:fs_cm}, and poses a good challenge for the elicitation task. 
Now, we validate the elicitation procedure with two real-world datasets.

The datasets are: (a) \textsc{Breast Cancer (BC)} Wisconsin Diagnostic dataset \cite{street1993nuclear} containing 569 instances, and  (b) \textsc{Magic (M)} dataset \cite{dvovrak2007softening} containing 19020 instances. For both the datasets, we standardize the attributes and split the data into two parts $\Scal_1$ and $\Scal_2$. On $\Scal_1$, we learn an estimator $\hat{\eta}$ using regularized logistic regression model with regularizing constant $\lambda=10$ and $\lambda=1$. We use $\Scal_2$ for making predictions and computing sample confusion matrices. 

We generated twenty eight different LPMs $\sphi$ by generating $\theta^*$ (or say, $\smmbf = (\cos{\theta}^*, \sin{\theta}^*))$. Fourteen from the first quadrant starting from $\pi/18$ radians to $5\pi/12$ radians in step of $\pi/36$ radians. Similarly, fourteen from the third quadrant starting from $19\pi/18$ to $17\pi/12$ in step of $\pi/36$ radians. We then use Algorithm~\ref{alg:linear} (Algorithm~\hyperlink{alg:quasiconvex}{2} for different tolerance $\epsilon$, for different datasets, and for different regularizing constant $\lambda$ in order to recover the estimate $\hat{\mmbf}$. We compute the error in terms of the proportion of the number of times when Algorithm \ref{alg:linear} (Algorithm~\hyperlink{alg:quasiconvex}{2}) failed to recover the true ${\smmbf}$ within $\epsilon$ threshold. 

\begin{table}[t]
	\caption{LPM elicitation results on real datasets ($\epsilon$ in radians). \textsc{M} and \textsc{BC} represent Magic and Breast Cancer dataset, respectively. $\lambda$ is the regularization parameter in the regularized logistic regression models. The table shows error in terms of the proportion of the number of times when Algorithm \ref{alg:linear} (Algorithm 2) failed to recover the true ${\smmbf} (\theta^*)$ within $\epsilon$ threshold. The observations made in the main paper are consistent for both the regularized models.}
	\label{tab:app:LPMreal}
	\begin{center}
		\begin{small}
			\begin{tabular}{|c|c|c|c|c|}
				\hline
				&
				\multicolumn{2}{|c|}{$\lambda = 10$} &  \multicolumn{2}{|c|}{$\lambda = 1$} \\ \hline
				$\epsilon$ & \textsc{M}  & \textsc{BC} & \textsc{M}  & \textsc{BC} \\
				\hline
				0.02 & 0.57 & 0.79  & 0.54 & 0.79 \\
				0.05 & 0.14  & 0.43  & 0.36 & 0.64 \\
				0.08 & 0.07  & 0.21  & 0.14 & 0.57 \\
				0.11 & 0.00  & 0.07  & 0.07 & 0.43 \\
				\hline
			\end{tabular}
		\end{small}
	\end{center}
\end{table}

We report our results in Table \ref{tab:app:LPMreal}. We see improved elicitation for dataset $M$, suggesting that ME improves with larger datasets. In particular, for dataset $M$, we elicit all the metrics within threshold $\epsilon = 0.11$ radians. We also observe that $\epsilon = 0.02$ is an overly tight tolerance for both the datasets leading to many failures. This is because the elicitation routine gets stuck at the closest achievable confusion matrix from finite samples, which need not be optimal within the given (small) tolerance. Furthermore, both of these observations are consistent for both the regularized  logisitic regression models with regularizer $\lambda$. 

Next, we discuss the case of LFPM elicitation. We use the same true metrics $\sphi$ as described in Section \ref{ssec:theoryexp} and follow the same process for eliciting LFPM, but this time we work with \textsc{MAGIC} dataset.  
In Table \ref{tab:app:LFPMtheoryreal} (columns 5, 6, and 7), we present the elicitation results on \textsc{MAGIC} dataset along with the mean $\alpha$ and the standard deviation $\sigma$ of the ratio of the elicited metric and the true metric. Again, for a better judgment, we show the function evaluation of the true metric and the elicited metric on the selected pairs of $(TP, TN) \in \partial\Ccal_+$ (used for Algorithm \ref{alg:grid-search}) in Figure \ref{fig:app:lfpm-real}, ordered by the parameter $\theta$. Although we do observe that the \emph{argmax} is different in two out of six cases (see Sub-figure \subref{fig:app:lfpm_2_magic} and Sub-figure \subref{fig:app:lfpm_3_magic}) due to finite sample estimation, elicited LFPMs are almost equivalent to the true metric up to a constant. 
\section{RELATED WORK}
\label{sec:discussion}

Our work may be compared to ranking from pairwise comparisons \cite{wauthier2013efficient}. However, we note that our results depend on novel geometric ideas on the space of confusion matrices. Thus, instead of a ranking problem, we show that ME in standard models can be reduced to just finding the maximizer (and minimizer) of an unknown function which in turn yields the true metric -- resulting in low query complexity. A direct ranking approach adds unnecessary complexity to achieve the same task. Further, in contrast to our approach, most large margin ordinal regression based ranking \cite{herbrich2000large} fail to control which samples are queried. There is another line of work, which actively controls the query samples for ranking, e.g., \cite{jamieson2011active}. However, to our knowledge, this requires that the number of objects is finite and finite dimensional -- thus cannot be directly applied to ME without significant modifications, e.g. 
exploiting confusion matrix properties, as we have. Learning a performance metric which correlates with human preferences has been studied before \cite{janssen2007meta, peyrard2017learning}; however, these studies learn a regression function over some predefined features which is fundamentally different from our problem. Lastly, while \cite{caruana2004data, ferri2009experimental} address how one might qualitatively choose between metrics, none addresses our central contribution -- a principled approach for eliciting the ideal metric from user feedback.

\section{CONCLUSION}
\label{sec:conclusion}
 
We conceptualize \emph{metric elicitation} and elicit linear and linear-fractional metrics using preference feedback over pairs of classifiers. We propose provably query efficient and robust algorithms which exploit key properties of the set of confusion matrices. 
In future, we plan to explore metric elicitation beyond binary classification.
\bibliographystyle{plain}
\bibliography{binaryMetrics}
	
\cleardoublepage
    
\renewcommand{\thesection}{\Alph{section}}
\setcounter{section}{0}
\newcommand{\pb}{\vspace*{-\parskip}\noindent\rule[0.5ex]{\linewidth}{1pt}}
\begin{appendices}

\section{Visualizing the Set of Confusion Matrices}
\label{appendix:visualization}
To clarify the geometry of the feasible set, we visualize one instance of the set of confusion matrices $\Ccal$ using the dual representation of the supporting hyperplanes. This contains the following steps. 
\benumerate[wide, labelwidth=!, labelindent=0pt]
\item \emph{Population Model:} We assume a joint probability for $\Xcal = [-1,1]$ and $\Ycal = \{0, 1\}$ given by
\begin{ceqn}
\begin{equation}
f_X = \Umbb[-1,1] \quad \text{and} \quad \eta(x) = \frac{1}{1 + e^{ax}},
\label{prob-dist}
\end{equation}
\end{ceqn}
where $\Umbb[-1,1]$ is the uniform distribution on $[-1, 1]$ and $a>0$ is a parameter controlling the degree of noise in the labels. If $a$ is large, then with high probability, the true label is $1$ on [-1, 0] and $0$ on [0, 1]. On the contrary, if $a$ is small, then there are no separable regions and the classes are mixed in $[-1,1]$.  

Furthermore, the integral $\int_{-1}^{1} \frac{1}{1 + e^{ax}}dx = 1$ for $a \in \Rmbb$ implying $ \Pmbb(Y = 1) = \zeta = \frac{1}{2} \; \forall \; a \in \Rmbb$.
\item \emph{Generate Hyperplanes:} Take $\theta \in [0, 2\pi]$ and set $\mmbf = (m_{11}, m_{00}) = (\cos\theta, \sin\theta)$. Let us denote $x'$ as the point where the probability of positive class $\eta(x)$ is equal to the optimal threshold of Proposition \ref{pr:bayeslinear}. Solving for $x$ in the equation $1/(1 + e^{ax}) = m_{00}/(m_{00} + m_{11})$ gives us
\begin{ceqn}
\begin{align}
x' &= \Pi_{[-1, 1]} \big\{\tfrac{1}{a}\ln\big(\tfrac{m_{11}}{m_{00}}\big)\big\},
\end{align}
\end{ceqn}
where $\Pi_{[-1,1]} \{z\}$ is the projection of $z$ on the interval $[-1,1]$. If $m_{11} + m_{00} \geq 0$, then the Bayes classifier $\hbar$ predicts class $1$ on the region $[-1, x']$ and $0$ on the remaining region. If $m_{11} + m_{00} < 0$, $\hbar$ does the opposite. Using the fact that $Y | X$ and $\hbar | X$ are independent, we have that
\begin{enumerate}[wide, labelwidth=!, labelindent=0pt]
	\item if $m_{11} + m_{00} \geq 0$, then 
	$$\oline{TP}_\mmbf = \frac{1}{2} \textstyle \int\limits_{-1}^{{x'}} \frac{1}{1 + e^{ax}}dx, \qquad \oline{TN}_\mmbf = \frac{1}{2} \int\limits_{{x'}}^{1} \frac{e^{ax}}{1 + e^{ax}}dx.$$
	\item if $m_{11} + m_{00} < 0$, then 
	$$\oline{TP}_\mmbf = \frac{1}{2} \textstyle\int\limits_{{x'}}^{1} \frac{1}{1 + e^{ax}}dx, \qquad \oline{TN}_\mmbf = \frac{1}{2} \int\limits_{-1}^{{x'}} \frac{e^{ax}}{1 + e^{ax}}dx.$$

Now, we can obtain the hyperplane as defined in \eqref{eq:support} for each $\theta$. 
We sample around thousand $\theta 's \in [0, 2\pi]$ 
randomly, obtain the hyperplanes following the above process, and plot them.
\end{enumerate}
The sets of feasible confusion matrices $\Ccal$'s for $ a = 0.5, 1, 2, 5, 10$, and $50$ are shown in Figure \ref{fig:fs_cm}. The middle white region is $\Ccal$: the intersection of the half-spaces associated with its supporting hyperplanes. 
The curve on the right corresponds to the confusion matrices on the upper boundary $\partial\Ccal_+$. Similarly, the curve on the left corresponds to the confusion matrices on the lower boundary $\partial\Ccal_-$. Points $(\zeta, 0) = (\frac{1}{2}, 0)$ and $(0, 1 - \zeta) = (0, \frac{1}{2})$ are the two vertices. The geometry is 180\degree rotationally symmetric around the point $(\frac{1}{4}, \frac{1}{4})$.

Notice that as we increase the separability of the two classes via $a$, all the points in $[0, \zeta] \times [0, 1-\zeta]$ becomes feasible. In other words, if the data is completely separable, then the corners on the top-right and the bottom left are achievable. If the data is `inseparable', then the feasible set contains only the diagonal line joining $(0,\frac{1}{2})$ and $(\frac{1}{2},0)$, which passes through $(\frac{1}{4},\frac{1}{4})$.   
\eenumerate

\begin{figure*}[t]
	\centering 
	\subfigure[a = 0.5]{
		{\includegraphics[width=5cm]{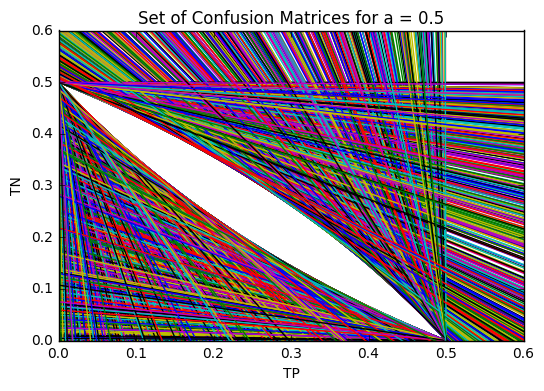}}
		\label{fig:cf_a_0_5}
	}
	\subfigure[a = 1]{
		{\includegraphics[width=5cm]{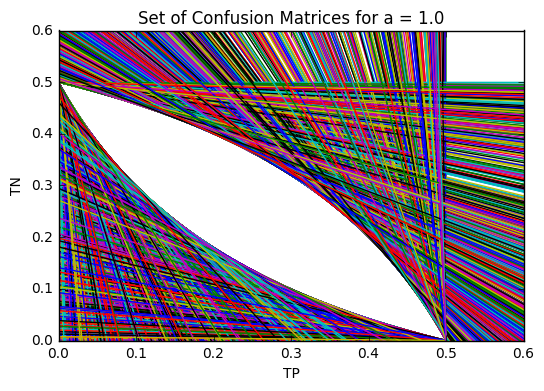}}
		\label{fig:cf_a_1}
	}
	\subfigure[a = 2]{
		{\includegraphics[width=5cm]{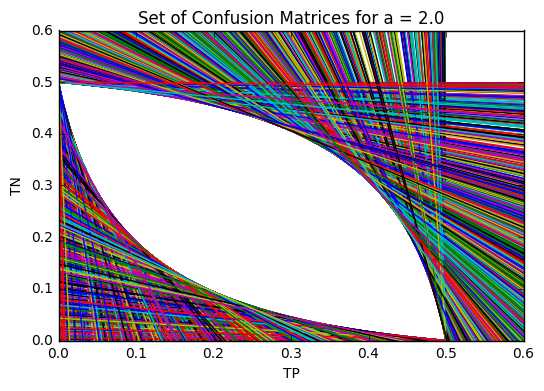}}
		\label{fig:cf_a_3}
	}
	\subfigure[a = 5]{
		{\includegraphics[width=5cm]{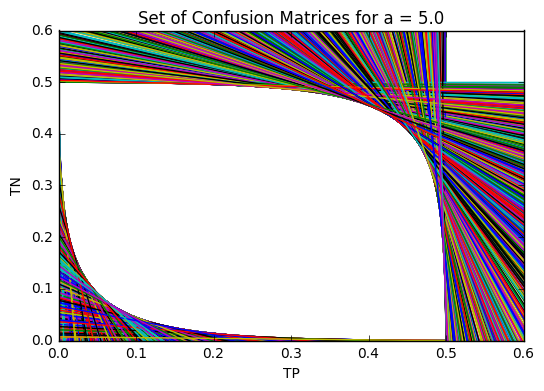}}
		\label{fig:cf_a_5}
	}
	\subfigure[a = 10]{
		{\includegraphics[width=5cm]{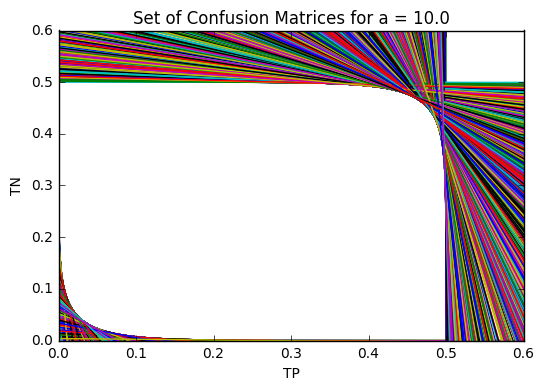}}
		\label{fig:cf_a_10}
	}
	\subfigure[a = 50]{
		{\includegraphics[width=5cm]{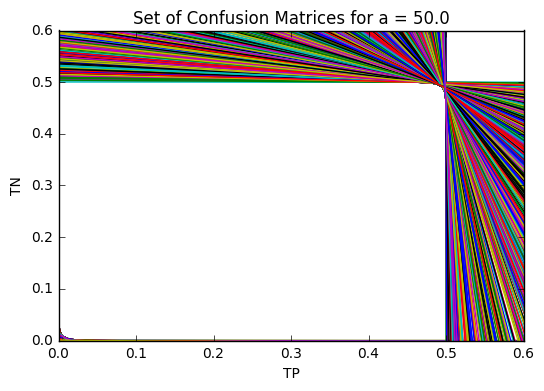}}
		\label{fig:cf_a_50}
	}
	\caption{Supporting hyperplanes and associated set of feasible confusion matrices for exponential model~\eqref{prob-dist} with $a = 0.5, 1, 2, 5, 10$ and $50$. The middle white region is $\Ccal$: the intersection of half-spaces associated with its supporting hyperplanes.}
	\label{fig:fs_cm}
	\end{figure*}


\pb

\section{Proofs}\label{appendix:proofs}
\newcommand{\tp}{TP}
\newcommand{\tn}{TN}
\newcommand{\df}{\,\mathrm df_X}
\newcommand{\dfz}{\,\mathrm df_Z}
\begin{lemma}\label{lem:properties-C}
The feasible set of confusion matrices $\mathcal C$ has the following properties:
	\renewcommand{\theenumi}{(\roman{enumi})}
	\begin{enumerate}[leftmargin=1cm]
	\item For all $(\tp,\tn)\in \mathcal C$, $0\leq \tp\leq \zeta$, and $0\leq \tn\leq 1-\zeta$.
	\item $(\zeta,0)\in\mathcal C$ and $(0,1-\zeta)\in \mathcal C$.
	\item For all $(\tp,\tn) \in \mathcal C$, $(\zeta-\tp, 1-\zeta-\tn)\in \mathcal C$.
	\item $\mathcal C$ is convex.
	\item $\mathcal C$ has a supporting hyperplane associated to every normal vector. 
	\item Any supporting hyperplane with positive slope is tangent to $\mathcal C$
	at $(0,1-\zeta)$ or $(\zeta,0)$.
	\end{enumerate}
\end{lemma}
\begin{proof} We prove the statements as follows: 

\begin{enumerate}[label=(\roman*)., leftmargin=1cm]
\item $0\leq \Pmbb[h=Y=1]\leq \Pmbb[Y=1]=\zeta$, and similarly, $0\leq \Pmbb[h=Y=0]\leq \Pmbb[Y=0]=1-\zeta$.

\item If $h$ is the trivial classifier which always predicts $1$, then  $\tp(h)=\Pr[h = Y=1] = \Pr[Y=1]=\zeta$, and $\tn(h)=0$. This means that $(\zeta, 0) \in \Ccal$. Similarly, if $h$ is the classifier which always predicts 0, then $\tp(h)=\Pr[h =Y=1] = 0$, and $\tn(h)=\Pr[h = Y=0] = \Pr[Y=0]= 1 - \zeta$. Therefore, $(0, 1 - \zeta) \in \Ccal$. 

\item Let $h$ be a classifier such that $\tp(h)=\tp$, $\tn(h)=\tn$. Now, consider the classifier $1 - h$ (which predicts exactly the opposite of $h$). We have that
\begin{ceqn}
\begin{align}
	\tp(1-h)&=\Pmbb[(1-h)=Y=1] \nonumber \\ &=\Pmbb[Y=1]-\Pmbb[h=Y=1] \nonumber \\
	&=\zeta-\tp(h). \nonumber
\end{align}
\end{ceqn}
A similar argument gives
$$\tn(1-h)=1-\zeta-\tn(h).$$

\item Consider any two confusion matrices $(\tp_1,\tn_1),\,(\tp_2,\tn_2)\in \mathcal C$, attained by the classifiers $h_1, h_2 \in \Hcal$, respectively. Let $0\leq \lambda\leq 1$. Define a classifier $h'$ which predicts the output from the classifier $h_1$ with probability $\lambda$ and predicts the output of the classifier $h_2$ with probability $1 - \lambda$. Then,
\begin{ceqn}
	\begin{align*}
		\tp(h')&=\Pmbb[h'=Y=1] \nonumber \\
		&=\Pmbb[h_1=Y=1|h=h_1]\Pmbb[h=h_1]  \nonumber \\
		&+ \Pmbb[h_2=Y=1|h=h_2]\Pmbb[h=h_2]\\
		&=\lambda\tp(h_1)+(1-\lambda)\tp(h_2).
	\end{align*}
\end{ceqn}
	A similar argument gives the convex combination for $\tn$. Thus, $\lambda(\tp(h_1),\tn(h_1)) +(1-\lambda)(\tp(h_2), \tn(h_2)) \in \Ccal$ and hence, $\Ccal$ is convex.
\item This follows from convexity (iv) and boundedness (i). 

\item For any bounded, convex region in $[0,\zeta]\times [0,1-\zeta]$ which contains the points $(0,\zeta)$ and $(0,1-\zeta)$, it is true that any positively sloped supporting hyperplane will be tangent to $(0,\zeta)$ or  $(0,1-\zeta)$.
\end{enumerate}
\end{proof}
\pb
\begin{lemma}\label{lem:max-at-bdy}	
	The boundary of $\mathcal C$ is exactly the confusion matrices of 
	estimators of the form $\lambda \1[\eta(x)\geq t] + (1-\lambda)\1[\eta(x)>t]$ and 
    $\lambda \1[\eta(x)< t] + (1-\lambda)\1[\eta(x)\leq t]$ for some $\lambda, t \in [0, 1]$.
\end{lemma}
\begin{proof}
To prove that the boundary is attained by estimators of these forms, consider solving the problem under the constraint $\Pmbb[h=1]=c$. 
We have $\Pmbb[h=1]=TP+FP$, and $\zeta=\Pmbb[Y=1]=TP+FN$, so we get
$$
	TP-TN\  =\  c + \zeta - TP - TN - FP - FN\  =\  c+\zeta - 1, 
$$
which is a constant. Note that no confusion matrix has two values of $TP- TN$. This effectively partitions $\Ccal$, 
since all confusion matrices are attained by varying $c$ from 0 to 1. 
Furthermore, since $A:= TN = TP - c - \zeta + 1$ is an affine space (a line in tp-tn coordinate system), $\mathcal C \cap A$ has at least one endpoint, because $A$ would pass through the box $[\zeta, 0] \times [0, 1- \zeta]$ and has at most two endpoints due to convexity and boundedness of $\Ccal$. 
Since $A$ is a line with positive slope, $\mathcal C\cap A$ is a single point only when $A$ is tangent to $\mathcal C$ at $(0,1-\zeta)$ or $(\zeta,0)$, from Lemma~\ref{lem:properties-C}, part (vi).

Since the affine space $A$ has positive slope, we claim that the two endpoints are attained by maximizing or minimizing $TP(h)$ subject to $\Pr[h=1]=c$.
It remains to show that this happens for estimators of the
form $h_{t+}^\lambda := {\lambda \1[\eta(x)\geq t]} + {(1-\lambda)\1[\eta(x)>t]}$ and 
$h_{t-}^\lambda:=\lambda \1[\eta(x)< t] + (1 - \lambda)\1[\eta(x)\leq t]$, respectively.

Let $h$ be any estimator, and recall
$$
	TP(h):=\int_{\mathcal X} \eta(x)\Pmbb[h=1|X=x]\df. 
$$
It should be clear that under a constraint $\Pmbb[h=1]=c$, the optimal choice of $h$ puts all the weight onto the larger values of $\eta$. One can begin by classifying those $X$ into the positive class where $n(X)$ is maximum, until one exhausts the budget of $c$. Let $t$ be such that $\Pmbb[h_{t+}^0=1]\leq c\leq \Pmbb[h_{t+}^1=1]$, and let $\lambda \in [0, 1]$ be chosen such that $\Pmbb[h_{t+}^\lambda=1]=c$,
then $h_{t+}^\lambda$ must maximize $TP(h)$ subject to $\Pmbb[h=1]=c$.

A similar argument shows that all TP-minimizing boundary points are attained by the $h_{t-}$'s.
\end{proof}

\bremark
Under Assumption \ref{as:eta}, $\1[\eta(x)>t] = \1[\eta(x)\geq t]$ and $\1[\eta(x)<t] = \1[\eta(x)\leq t]$. Thus, the boundary of $\mathcal C$ is the confusion matrices of estimators of the form $\1[\eta(x)\geq t]$  and $\1[\eta(x)\leq t]$ for some $t \in [0, 1]$.
\eremark
\pb
\begin{proof}[Proof of Proposition~\ref{pr:bayeslinear}]
	\textit{``Let $\phi \in \varphi_{LPM}$, then 
$$
\hbar(x) = \left\{\begin{array}{lr}
			 \1[\eta(x) \geq \frac{m_{00}}{m_{11} + m_{00}}],& \; m_{11} + m_{00} \geq 0 \\
			 \1[\frac{m_{00}}{m_{11} + m_{00}} \geq \eta(x) ],& \;  o.w. 
	 	 \end{array}\right\}
$$
is a Bayes optimal classifier \emph{w.r.t} 
$\phi$. Further, the inverse Bayes classifier is given by $\barbelow{h}= 1 - \hbar$.''}
    
Note, we are maximizing a linear function on a convex set. There are 6 cases to consider:
    \begin{enumerate}[leftmargin=0.5cm]
    \item If the signs of $m_{11}$ and $m_{00}$ differ, the maximum is attained either at $(0,1-\zeta)$ 
    or $(\zeta,0)$, as per Lemma~\ref{lem:properties-C}, part (vi). 
    Which of the two is optimum depends on whether $|m_{11}|\geq |m_{00}|$, i.e. on the sign of
    $m_{11}+m_{00}$. It should be easy to check that in all 4 possible cases, the statement holds,
    noting that in all 4 cases, $0 \leq m_{00}/(m_{11}+m_{00}) \leq 1.$
    \item If $m_{11},m_{00}\geq 0$, then the maximum is attained on $\partial\mathcal{C}_+$, and the
    proof below gives the desired result.
    
	We know, from Lemma~\ref{lem:max-at-bdy}, that $\hbar$ must be of the form
	$\1[\eta(x)\geq t]$ for some $t$. It suffices to find $t$. 
	Thus, we wish to maximize $m_{11}TP(h_t)+m_{00}TN(h_t)$.
	Now, let $Z:=\eta(X)$ be the random variable obtained by evaluating $\eta$ at random $X$. Under Assumption~\ref{as:eta}, $df_X = df_Z$ and we have that 
	$$
		TP(h_t)\ = \int_{x:\eta(x)\geq t} \eta(x) \df\ =
		\int_{t}^1 z \dfz 
	$$
	Similarly, $\TN(h_t) = \int_0^t (1-z)\dfz$. Therefore, 
	\begin{ceqn}
	\begin{align}
		\tfrac{\partial}{\partial t} \big(m_{11}&TP(h_t)+m_{00}
		TN(h_t)\big)\nonumber \\  
		&= -m_{11}tf_Z(t) + \cdot m_{00}(1-t)f_Z(t). \nonumber
	\end{align}
    \end{ceqn}
	So, the critical point is attained at $t=m_{00}/(m_{11}+m_{00})$, as desired.
	A similar argument gives the converse result for $m_{11} + m_{00}< 0$.
	\item if $m_{11},m_{00}<0$, then the maximum is attained on $\partial\mathcal{C}_-$, and an
    argument identical to the proof above gives the desired result. 
	\eenumerate
\end{proof}
\pb
\begin{proof}[Proof of Proposition~\ref{pr:strict-convex}]
	\textit{``The set of confusion matrices $\Ccal$ is convex, closed, contained in the rectangle $[0,\zeta]\times [0,1-\zeta]$ (bounded), and $180\degree$ rotationally symmetric around the center-point $(\frac{\zeta}{2}, \frac{1-\zeta}{2})$. Under Assumption~\ref{as:eta}, $(0,1-\zeta)$ and $(\zeta,0)$ are the only vertices of $\Ccal$, and $\Ccal$ is strictly convex. Thus, any supporting hyperplane of $\Ccal$ is tangent at only one point.''}
    
    
    That $\Ccal$ is convex and bounded is already proven in Lemma~\ref{lem:properties-C}.
    To see that $\mathcal C$ is closed, note that, from Lemma~\ref{lem:max-at-bdy},
    every boundary point is attained. From Lemma~\ref{lem:properties-C}, part (iii), it follows that $\Ccal$ is $180\degree$ rotationally symmetric around the point $(\frac{\zeta}{2}, \frac{1-\zeta}{2})$. 
    
    Further, recall every boundary point of $\mathcal C$ can be attained by a thresholding estimator. By the discussion in Section~\ref{sec:confusion}, every boundary point is the optimal classifier for some linear performance metric, and the vector defining this linear metric is exactly the normal vector of the supporting hyperplane at the boundary point.
    
    A vertex exists if (and only if) some point is supported by more than one tangent hyperplane in two dimensional space. This means it is optimal for more than one linear metric. Clearly, all the hyperplanes corresponding to the slope of the metrics where $m_{11}$ and $m_{00}$ are of opposite sign (i.e. hyperplanes with positive slope) support either $(\zeta, 0)$ or $(0, 1-\zeta)$. So, there are at least two supporting hyperplanes at these points, which make them the vertices. Now, it remains to show that there are no other vertices for the set $\Ccal$. 
    
    Now consider the case when the slopes of the hyperplanes are negative, i.e. $m_{11}$ and $m_{00}$ have the same sign for the corresponding linear metrics. We know from Proposition~\ref{pr:bayeslinear} that optimal classifiers for linear metrics are threshold classifiers. Therefore there exist more than one threshold classifier of the form $h_t = \1[\eta(x)\geq t]$ with the same confusion matrix. Let's call them $h_{t_1}$ and $h_{t_2}$ for the two thresholds $t_1, t_2 \in [0, 1]$. This means that $\int_{x: \eta(x) \geq t_1} \eta(x)df_X = \int_{x: \eta(x) \geq t_2} \eta(x)df_X$. Hence, there are multiple values of $\eta$ which are never attained! This contradicts that $g$ is strictly decreasing. Therefore,  there are no vertices other than $(\zeta, 0)$ or $(0, 1-\zeta)$ in $\Ccal$. 
    
    Now, we show that no supporting hyperplane is tangent at multiple points (i.e., there no flat regions on the boundary). If suppose there is a hyperplane which supports two points on the boundary. Then there exist two threshold classifiers with arbitrarily close threshold values, but confusion matrices that are well-separated. Therefore, there must exist some value of $\eta$ which exists with non-zero probability, contradicting the continuity of $g$. 
    By the discussion above, we conclude that under Assumption~\ref{as:eta}, every supporting hyperplane to the convext set $\Ccal$ is tangent to only one point. This makes the set $\Ccal$ strictly convex.
\end{proof}
\pb
\begin{proof}[Proof of Lemma~\ref{lem:quasiconcave}] \textit{``Let 
$\rho^+:[0,1]\to \partial\mathcal C_+$, $\rho^-:[0,1]\to \partial\mathcal C_-$
be continuous, bijective, parametrizations of the upper and lower boundary,
respectively. Let $\phi:\mathcal C\to \mathbb R$ be a quasiconcave function,
and $\psi:\mathcal C\to \mathbb R$ be a quasiconvex function, 
which are 
monotone increasing in both $TP$ and $TN$.
Then the composition $\phi\circ \rho^+: [0,1]\to\mathbb R$ is
quasiconcave (and therefore unimodal) on the interval $[0, 1]$, and $\psi\circ\rho^-:[0,1]\to\mathbb R$ is quasiconvex (and therefore unimodal) on the interval $[0,1]$.'' }

We will prove the result for $\phi\circ \rho^+$ on $\partial\mathcal C^+$, and the argument for $\psi\circ \rho^-$ on $\partial\mathcal C^+$ is essentially the same. For simplicity, we drop
the $+$ symbols in the notation. Recall that a function is quasiconcave if and only if its superlevel sets are convex. 

It is given that $\phi$ is quasiconcave. Let $S$ be some superlevel set of $\phi$. We first want to show that for any $r<s<t$, if $\rho(r)\in S$ and $\rho(t)\in S$,
then $\rho(s)\in S$. Since $\rho$ is a continuous bijection, due to the geometry of $\Ccal$ (Lemma~\ref{lem:properties-C} and Proposition~\ref{pr:strict-convex}), we must have --- without loss of generality ---
$TP(\rho(r))< TP(\rho(s)) < TP(\rho(t))$, and $TN(\rho(r))>TN(\rho(s))>TN(\rho(t))$.
(otherwise swap $r$ and $t$). Since the set $\Ccal$ is strictly convex and the image of 
$\rho$ is $\partial \mathcal C$, then $\rho (s)$ must dominate (component-wise) a point in the convex combination
of $\rho(r)$ and $\rho(t)$. Say that point is $z$. Since $\phi$ is monotone increasing, then $x\in S\implies y \in S$ for all $y\geq x$ componentwise. Thereofore, $\phi(\rho(s)) \geq \phi(z)$. Since, $S$ is convex, $z \in S$ and, due to the argument above, $\rho(s) \in S$.

This implies that $\rho^{-1}(\partial \mathcal C\cap S)$ is an interval, and is therefore convex. Thus, the superlevel sets of $\phi\circ \rho$ are convex, so it is quasiconcave,
as desired. This implies unimodaltiy as a function over the real line which has more than one local maximum can not be quasiconcave (consider the super-level set for some value slightly less than the lowest of the two peaks).
\end{proof}
\pb
\begin{proof}[Proof of Proposition~\ref{prop:sufficient}] \textit{``Sufficient conditions for $\phi \in \varphi_{LFPM}$ to be bounded in $[0, 1]$ and simultaneously monotonically increasing in TP and TN are: $p_{11}, p_{00} \geq 0$, $p_{11} \geq q_{11}$, $p_{00} \geq q_{00}$, $q_0 = (p_{11} - q_{11})\zeta + (p_{00} - q_{00})(1 - \zeta) + p_0$, $p_0 = 0$, and $p_{11} + p_{00} = 1$ (Conditions in Assumpotion~\ref{assump:sufficient}). WLOG, we can take both the numerator and denominator to be positive.'' }

For this proof, we denote $TP$ and $TN$ as $C_{11}$ and $C_{00}$, respectively. Let us take a linear-fractional metric
\begin{ceqn}
\begin{align}
\phi(C) = \frac{p_{11}C_{11}+p_{00}C_{00}+p_0}{q_{11}C_{11}+q_{00}C_{00}+q_0}
\label{eq:linear-f}
\end{align}
\end{ceqn}
where $p_{11}, q_{11},p_{00},q_{00}$ are not zero simultaneously. 
We want $\phi(C)$ to be monotonic in TP, TN and bounded. If for any $C \in \Ccal$, $\phi(C) < 0$, we can add a large positive constant such that $\phi(C) \geq 0$, and still the metric would remain linear fractional. So, it is sufficient to assume $\phi(C) \geq 0$. Furthermore, boundedness of $\phi$ implies $\phi(C) \in [0,D]$, for some $ \Rmbb \ni D \geq 0$. Therefore, we may divide $\phi(C)$ by $D$ so that $\phi(C) \in [0,1]$ for all $C \in \Ccal$. Still, the metric is linear fractional and $\phi(C) \in [0,1]$.

Taking derivative of $\phi(C)$ w.r.t. $C_{11}$.
\begin{ceqn}
\begin{align*}
\frac{\partial \phi(C)}{\partial C_{11}} &= \frac{p_{11}}{q_{11}C_{11}+q_{00}C_{00}+q_0} \nonumber \\
&-\frac{q_{11}(p_{11}C_{11}+p_{00}C_{00}+p_0)}{(q_{11}C_{11}+q_{00}C_{00}+q_0)^2} \geq 0 
\end{align*}
\end{ceqn}

\begin{ceqn}
\begin{align*}
\Rightarrow p_{11}(q_{11}C_{11}+q_{00}C_{00}+q_0) \geq q_{11}(p_{11}C_{11}+p_{00}C_{00}+p_0)
\end{align*}
\end{ceqn}
If denominator is positive then the numerator is positive as well.
\begin{itemize}
\item Case 1: The denominator $q_{11}C_{11}+q_{00}C_{00}+q_0 \geq 0$.
\begin{itemize}
\item Case (a) $q_{11} > 0$. 
\begin{ceqn}
\begin{align*}
\Rightarrow p_{11} &\geq q_{11} \phi(C) \\
\Rightarrow p_{11} &\geq q_{11}\sup_{C\in \Ccal} \phi(C)\\
\Rightarrow p_{11} &\geq q_{11}\btau \qquad \text{ (Necessary Condition)}
\end{align*}
\end{ceqn}
We are considering sufficient condition, which means $\btau$ can vary from $[0, 1]$. Hence, a sufficient condition for monotonicity in $C_{11}$ is $p_{11} \geq q_{11}$. Furthermore,
$p_{11} \geq 0$ as well.
\item Case (b) $q_{11} < 0$.
\begin{ceqn}
\begin{align*}
\Rightarrow p_{11} &\geq {q_{11}} \btau
\end{align*}
\end{ceqn}
Since $q_{11} <0$ and $\btau \in [0,1]$, sufficient condition is $p_{11} \geq 0$. So, in this case as well we have that
\begin{ceqn}
\begin{align*}
p_{11} \geq q_{11}, ~p_{11} \geq 0.
\end{align*}
\end{ceqn}
\item Case(c) $q_{11} = 0$.
\begin{ceqn}
\begin{align*} 
\Rightarrow p_{11} &\geq 0
\end{align*}
\end{ceqn}
We again have $p_{11}\geq q_{11}$ and $p_{11} \geq 0$ as sufficient conditions. 

A similar case holds for $C_{00}$, implying $p_{00} \geq q_{00}$ and $p_{00} \geq 0$.
\end{itemize}
\item Case 2: The denominator $q_{11}C_{11}+q_{00}C_{00} + q_0$ is negative. 
\begin{ceqn}
\begin{align*}
p_{11} &\leq q_{11} \Big(\frac{p_{11}C_{11}+p_{00}C_{00}+p_0}{q_{11}C_{11}+q_{00}C_{00}+q_0}\Big)\\
\Rightarrow p_{11} &\leq q_{11} \btau
\end{align*}
\end{ceqn}
\begin{itemize}
\item Case(a) If $q_{11} > 0$. So, we have $p_{11} \leq q_{11}$ and $p_{11} \leq 0$  as sufficient condition.
\item Case(b) If $q_{11} < 0$, $\Rightarrow p_{11} \leq q_{11}$. So, we have $q_{11} < 0$, $\Rightarrow p_{11} <0$ as sufficient condition.
\item Case(c) If $q_{11} = 0$, $\Rightarrow p_{11} \leq 0$ and $p_{11} \leq q_{11}$ as sufficient condition.

So in all the cases we have that 
\begin{ceqn}
\begin{align*}
p_{11} \leq q_{11} &\text{ and } p_{11} \leq 0\\
\end{align*}
\end{ceqn}
as the sufficient conditions. A similar case holds for $C_{00}$ resulting in $p_{00} \leq q_{00}$ and $p_{00} \leq 0$. 
\end{itemize}
\end{itemize}

Suppose the points where denominator is positive is $\Ccal^{+}\subseteq \Ccal$. Suppose the points where denominator is negative is $\Ccal^{-} \subseteq \Ccal$. For gradient to be non-negative at points belonging to $\Ccal^{+}$, the sufficient condition is 
\begin{ceqn}
\begin{align*}
p_{11} \geq q_{11} &\text{ and } p_{11} \geq 0\\
p_{00} \geq q_{00} &\text{ and } p_{00} \geq 0
\end{align*}
\end{ceqn}
For gradient to be non-negative at points belonging to $\Ccal^{-}$, the sufficient condition is
\begin{ceqn}
\begin{align*}
p_{11} \leq q_{11} &\text{ and } p_{11} \leq 0\\
p_{00} \leq q_{00} &\text{ and } p_{00} \leq 0
\end{align*}
\end{ceqn}
If $\Ccal_{+}$ and $\Ccal_{-}$ are not empty sets, then the gradient is non-negative only when $p_{11}, p_{00} = 0$ and $q_{11}, q_{00} = 0$. This is not possible by the definition described in \eqref{eq:linear-f}. Hence, one of $\Ccal_{+}$ or $\Ccal_{-}$ should be empty. WLOG, we assume $\Ccal_{-}$ is empty and conclude that $\Ccal_{+} = \Ccal$. \\
An immediate consequence of this is, WLOG, we can take both the numerator and the denominator to be positive, and the sufficient conditions for monotonicity are as follows:
\begin{ceqn}
\begin{align*}
p_{11} \geq q_{11} \text{ and } p_{11} \geq 0\nonumber\\
p_{00} \geq q_{00} \text{ and } p_{00} \geq 0
\end{align*}
\end{ceqn}

Now, let us take a point in the feasible space $(\zeta,0)$. We know that
\begin{ceqn}
\begin{align}
\phi((\zeta,0)) &= \frac{p_{11}\zeta+p_0}{q_{11}\zeta + q_0} \leq \btau \nonumber \\ 
&\Rightarrow p_{11}\zeta + p_0 \leq \btau (q_{11}\zeta + q_0) \nonumber\\
&\Rightarrow (p_{11} - \btau q_{11})\zeta + (p_0 - \btau q_0) \leq 0\nonumber \\
&\Rightarrow (p_0 - \btau q_0) \leq -\underbrace{(p_{11}-\btau q_{11})}_{\text{positive}}\underbrace{\zeta}_{\text{positive}} \nonumber \\
&\Rightarrow (p_0 - \btau q_0) \leq 0.
\label{eq:p0zero}
\end{align}
\end{ceqn}
Metric being bounded in $[0,1]$ gives us 
\begin{ceqn}
\begin{align*}
\frac{p_{11}C_{11}+p_{00}C_{00}+p_0}{q_{11}C_{11}+q_{00}C_{00}+q_0} & \leq 1 \\
\Rightarrow p_{11}C_{11} + p_{00}C_{00} + p_0 &\leq q_{11}C_{11} + q_{00}C_{00} + q_0 
\end{align*}
\end{ceqn}
$$
\Rightarrow q_0 \geq (p_{11}-q_{11})c_{11} + (p_{00}-q_{00})c_{00} + p_0 \qquad \forall C \in \Ccal.
$$
Hence, a sufficient condition is 
$$q_0 = (p_{11}-q_{11})\zeta + (p_{00}-q_{00})(1-\zeta) + p_0.$$
Equation \eqref{eq:p0zero}, which we derived from monotonicity, implies that
\begin{itemize}[leftmargin=0.5cm]
\item Case (a) $q_0 \geq 0$, $\Rightarrow p_0 \leq 0$ as a sufficient condition. 
\item Case (b) $q_0 \leq 0$, $\Rightarrow p_0 \leq q_0 \leq 0$ as a sufficient condition. 
\end{itemize}
Since the numerator is positive for all $C \in \Ccal$ and $p_{11}, p_{00} \geq 0$, a sufficient condition for $p_0$ is $p_0 = 0$.

Finally, a monotonic, bounded in $[0,1]$, linear fractional metric is defined by
\begin{ceqn}
\begin{align*}
\phi(C) &= \frac{p_{11}c_{11}+p_{00}c_{00}+p_0}{q_{11}c_{11}+q_{00}c_{00}+q_0},
\end{align*}
\end{ceqn}
where $p_{11} \geq q_{11}, p_{11} \geq 0,
p_{00} \geq q_{00}, p_{00} \geq 0,
q_0 = (p_{11}-q_{11})\zeta + (p_{00}-q_{00})(1-\zeta) + p_0,
p_0 = 0$, and $p_{11}, q_{11}, p_{00}$, and $q_{00}$ are not simulataneously zero. Further, we can divide the numerator and denominator with $p_{11} + p_{00}$ without changing the metric $\phi$ and the above sufficient conditions. Therefore, for elicitation purposes, we can take $p_{11} + p_{00} = 1$.
\end{proof}
\begin{proof}[Proof of Proposition~\ref{pr:solvesystem}] \textit{``Under Assumption~\ref{assump:sufficient}, knowing $p_{11}'$ solves the system of equations~\eqref{eq:lin-fr-equi} as follows:
\begin{ceqn}
\begin{align}
	p_{00}' &= 1 - p_{11}', \, \nonumber q_0' = \oline{C}_0\frac{P'}{Q'}, \nonumber \\
	q_{11}' &= (p_{11}' - \bmone)\frac{P'}{Q'}, \, q_{00}' = (p_{00}' - \bmzero)\frac{P'}{Q'}, 
\end{align}
\end{ceqn}
where $P' = p_{11}'\zeta + p_{00}'(1 - \zeta)$ and $Q' = P' + \oline{C}_0 - \bmone\zeta -  \bmzero(1 - \zeta)$. Thus, it elicits the LFPM.'' }

For this proof as well, we use $TP = C_{11}$ and $TN = C_{00}$. Since the linear fractional matrix is monotonically increasing in $C_{11}$ and $C_{00}$, it is maximized at the upper boundary $\partial \Ccal_+$. Hence $m_{11} \geq 0$ and $m_{00} \geq 0$. So, after running Algorithm \ref{alg:linear}, we get a hyperplane such that 
\begin{ceqn}
\begin{align}
p_{11} - \tau q_{11} &= \alpha m_{11}, \quad
p_{00} - \tau q_{00} = \alpha m_{00}, \nonumber \\ 
p_0 -\tau q_0 &= -\alpha\underbrace{(m_{11}C_{11}^* + m_{00}C_{00}^*)}_{=: C_0}.
\label{eq:system}
\end{align}
\end{ceqn}
Since $p_{11} - \btau q_{11} \geq 0$ and $m_{11} \geq 0$, $\Rightarrow \alpha \geq 0$. As discussed in the main paper, we avoid the case when $\alpha = 0$. Therefore, we have that $\alpha > 0$.

Equation \eqref{eq:system} implies that 
\begin{ceqn}
\begin{align*}
\frac{p_{11}}{\alpha} - \frac{\tau q_{11}}{\alpha} &= m_{11}, \quad
\frac{p_{00}}{\alpha} - \frac{\tau q_{00}}{\alpha} = m_{00}, \nonumber \\
\frac{p_0}{\alpha} - \frac{\tau q_0}{\alpha} &= -C_0.
\end{align*}
\end{ceqn}
Assume $p_{11}' = \frac{p_{11}}{\alpha}, p_{00}' = \frac{p_{00}}{\alpha}$, $q_{11}' = \frac{q_{11}}{\alpha}$, $q_{00}' = \frac{q_{00}}{\alpha}$, $p_0' = \frac{p_0}{\alpha}$, $q_0' = \frac{q_0}{\alpha}$. Then, the above system of equations turns into
\begin{ceqn}
\begin{align*}
p_{11}' - \btau q_{11}' &= m_{11}, \quad
p_{00}' - \btau q_{00}' = m_{00}, \nonumber \\
p_0' - \btau q_0' &= -C_0.
\end{align*}
\end{ceqn}
A $\phi'$ metric defined by the $\ppone, \ppzero, \pqone, \pqzero, \pqnot$ is monotonic, bounded in $[0,1]$, and satisfies all the sufficient conditions of Assumptions~\ref{assump:sufficient}, i.e.,
\begin{ceqn}
\begin{align*}
p_{11}' \geq q_{11}' ~,~ p_{00}' \geq q_{11}',~
p_{11}' \geq 0 ~,~ p_{00}' \geq 0, \nonumber \\
q_0' = (p_{11}' - q_{11})\pi + (p_{00}' - q_{00}')\pi + p_0', ~
p_0' = 0.
\end{align*}
\end{ceqn}
As discussed in the main paper, solving the above system does not harm the elicitation task. For simplicity, replacing the `` $'$ " notation with the normal one, we have that
\begin{ceqn}
\begin{align*}
p_{11} - \btau q_{11} &= m_{11}, \quad
p_{00} - \btau q_{00} = m_{00}, \nonumber \\
p_0 - \btau q_0 &= -C_0
\end{align*}
\end{ceqn}
From last equation, we have that $\btau  = \frac{C_0 + p_0}{q_0}$. Putting it in the rest gives us
\begin{ceqn}
\begin{align*}
q_0 p_{11} - (C_0 + p_0)q_{11} = m_{11}q_0, \nonumber \\
q_0 p_{00} - (C_0 + p_0) q_{00} = m_{00} q_0.
\end{align*}
\end{ceqn}
We already have
\begin{ceqn}
\begin{align*}
q_0 &= (p_{11}-q_{11})\zeta + (p_{00}-q_{00})(1-\zeta)+p_0\\
\Rightarrow q_{11} &= \frac{p_{00}(1-\zeta)-q_{00}(1-\zeta) + p_{11}\zeta - q_0+p_0}{\zeta},
\end{align*}
\end{ceqn}
which further gives us
\begin{ceqn}
\begin{align*}
q_0 &= \frac{(C_0 + p_0)[p_{00}(1-\zeta)+p_{11}\zeta +  p_0]}{p_{11}\zeta + p_{00}(1-\zeta) + p_0 + C_0 - m_{11}\zeta - m_{00}(1-\zeta)},\\
q_{00} &= \frac{(p_{00}-m_{00})[p_{00}(1-\zeta)+p_{11}\zeta + p_0]}{p_{11}\zeta + p_{00}(1-\zeta) + p_0 + C_0 -m_{11}\zeta -m_{00}(1-\zeta)},\\
q_{11} &= \frac{(p_{11}-m_{11})[p_{00}(1-\zeta) + p_{11}\zeta + p_0]}{p_{11}\zeta + p_{00}(1-\zeta) + p_0 + C_0 - m_{11}\zeta -m_{00}(1-\zeta)}.
\end{align*}
\end{ceqn}
Define
\begin{ceqn}
\begin{align*}
P &:= p_{00}(1-\zeta) + p_{11}\zeta + p_0,\\
Q &:= P + C_0 - m_{11}\zeta -m_{00}(1-\zeta).
\end{align*}
\end{ceqn}
Hence, 
\begin{ceqn}
\begin{align*}
q_0 &= (C_0 + p_0)\frac{P}{Q}, \quad
q_{11} = (p_{11}-m_{11})\frac{P}{Q}, \nonumber \\
q_{00} &= (p_{00}-m_{00})\frac{P}{Q}.
\end{align*}
\end{ceqn}
Now using sufficient conditions, we have $p_0 = 0$. The final solution is the following:
\begin{ceqn}
\begin{align*}
q_0 &= C_0 \frac{P}{Q}, \quad
q_{11} = (p_{11} - m_{11})\frac{P}{Q}, \nonumber \\
q_{00} &= (p_{00} - m_{00})\frac{P}{Q},
\label{eq:systemsolve}
\end{align*}
\end{ceqn}
where $P:= p_{11}\zeta + p_{00}(1-\zeta) $ and $Q:= P + C_0 - m_{11}\zeta - m_{00}(1-\zeta)$. 
We have taken ${p}_{11} + {p}_{00} = 1$, but the original $p'_{11} + p'_{00} = \frac{1}{\alpha}$. Therefore, we learn $\hat{\phi}(C)$ such that such that $\hat{\phi}(C) = \alpha \phi(C)$.
\end{proof}
\pb
\bcorollary \label{cor:f-beta} 
For $F_\beta$-measure, where $\beta$ is unknown, Algorithm \ref{alg:linear} elicits the true performance metric up to a constant in $O(\log(\frac{1}{\epsilon}))$ queries to the oracle.
\ecorollary
\begin{proof}
Algorithm \ref{alg:linear} gives us the supporting hyperplane, the trade-off, and the Bayes confusion matrix. If we know $p_{11}$, then we can use Proposition \ref{pr:solvesystem} to compute the other coefficients. In $F_\beta$-measure, $p_{11}=1$, and we do not require Algorithms \hyperlink{alg:quasiconvex}{2} and \ref{alg:grid-search}.
\end{proof}
\pb
\begin{proof}[Proof of Theorem~\ref{thm:quasi}] \textit{``Given $\epsilon,\epsilon_\Omega \geq 0$ and a 1-Lipschitz metric $\phi$ that is monotonically increasing in TP, TN. If it is quasiconcave (quasiconvex) then Algorithm \ref{alg:linear} (Algorithm~\hyperlink{alg:quasiconvex}{2}) finds an approximate maximizer $\Cbar$ (minimizer $\barbelow{C}$). Furthemore, $(i)$ the algorithm returns the supporting hyperplane at that point, $(ii)$ the value of $\phi$ at that point is within $O(\sqrt{\epsilon_\Omega} +  \epsilon)$ of the optimum, and $(iii)$ the number of queries is $O(\log\frac1\epsilon)$.'' }

\begin{enumerate}[leftmargin=0.5cm, label=(\roman*)]
\item As a direct consequence of our representation of the points on the boundary via their supporting hyperplanes (Section~\ref{ssec:parametrization}), when we search for the maximizer (mimimizer), we also get the associated supporting hyperplane as well.

\item  By the nature of binary search, we are effectively narrowing our search interval around some target angle $\theta_0$. Furthermore, since the oracle queries are correct unless the $\phi$ values are within $\epsilon_\Omega$, we must have $|\phi(C_{\oline \theta})-\phi(C_{\theta_0})|<\epsilon_\Omega$, and we output $\theta'$ such that $|\theta_0-\theta'|<\epsilon$. Now, we want to check the bound $|\phi(C_{\theta'}) - \phi(C_{\oline{\theta}})|$. In order to do that, we will also consider the threshold corresponding to the supporting hyperplanes at $C_\theta$'s, i.e. $\delta_\theta = \sfrac {\sin\theta}{\sin\theta + \cos\theta}$. 

Notice that,
\begin{ceqn}
\begin{align}
|\phi(C_{\oline{\theta}}) - \phi(C_{\theta'})| &= |\phi(C_{\oline{\theta}}) -\phi(C_{\theta_0}) \nonumber \\
&\qquad+ \phi(C_{\theta_0}) - \phi(C_{\theta'})| \nonumber\\
&\leq |\phi(C_{\oline{\theta}}) -\phi(C_{\theta_0})|\nonumber \\
&\qquad + |\phi(C_{\theta_0}) - \phi(C_{\theta'})| 
\end{align}
\end{ceqn}
The first term is bounded by $\epsilon_{\Omega}$ due to the oracle assumption.  For the bounds the second term, consider the following.
$$
|TP(C_{\theta_0}) - TP(C_{\theta'})|
$$
\begin{ceqn}
\begin{align}
&= \left|\int\limits_{x:\frac{sin\theta_0}{sin\theta_0 + cos\theta_0}\geq\eta(x)\geq\frac{sin\theta'}{sin\theta' + cos\theta'}}\!\!\!\!\!\!\!\!\!\!\!\! \eta(x)\df\right| \nonumber \\
&\leq \left|\int\limits_{x:\frac{sin\theta_0}{sin\theta_0 + cos\theta_0} - \oline{\delta}\geq\eta(x) - \oline{\delta}\geq\frac{sin\theta'}{sin\theta' + cos\theta'}- \oline{\delta}}\!\!\!\!\!\!\!\!\!\!\!\! \df\right| \nonumber \\
&= \left|\int\limits_{x:\frac{sin\theta_0}{sin\theta_0 + cos\theta_0} - \frac{sin\oline{\theta}}{sin\oline{\theta} + cos\oline{\theta}} \geq\eta(x) - \oline{\delta}\geq\frac{sin\theta'}{sin\theta' + cos\theta'}- \frac{sin\oline{\theta}}{sin\oline{\theta} + cos\oline{\theta}}}\!\!\!\!\!\!\!\!\!\!\!\! \df\right| \nonumber \\
&= \left|\int\limits_{x:\frac{sin(\theta_0 - \oline{\theta})}{sin(\theta_0 + \oline{\theta})  + cos(\theta_0 - \oline{\theta})} \geq\eta(x) - \oline{\delta}\geq\frac{sin\theta'}{sin\theta' + cos\theta'}- \frac{sin\oline{\theta}}{sin\oline{\theta} + cos\oline{\theta}}}\!\!\!\!\!\!\!\!\!\!\!\! \df\right|,  
\label{eq:integrals}
\end{align}
\end{ceqn}

where the inequality in the second step follows from the fact that $\eta(x) \leq 1$. 

Recall that the left term in the integral limits is actually, $\delta_{\theta_0} - \delta_{\oline\theta}$. When $|\phi(C_{\delta_{\theta_0}})-\phi(C_{\delta_{\oline\theta}})|<\epsilon_\Omega$, then we have $|\oline\delta-\delta_0|<\frac 2{k_0}\sqrt{
k_1\epsilon_\Omega}$. The proof of this statement is given in the proof of Theorem~\ref{thm:linear} (proved later).
Since sin is 1-Lipschitz, adding and subtracting $\sin\theta_0/(\sin\theta_0 + \cos\theta_0)$ in the right term of the integration limit gives us the minimum value of the right term to be $-\epsilon-\frac{2\sqrt{k_1\epsilon_\Omega}}{k_0}$.  
This implies that the quantity in ~\eqref{eq:integrals} is less than

\begin{ceqn}
\begin{align}
&\Pmbb[\{(\eta(X) - \oline{\delta}) \leq \frac{2}{k_0}\sqrt{k_1\epsilon_\Omega}\} \cap \nonumber \\ &\qquad\{(\oline{\delta} - \eta(X)) \leq \epsilon + \frac{2}{k_0}\sqrt{k_1\epsilon_\Omega}\}]\nonumber \\
&\leq\Pmbb[(\oline{\delta} - \eta(X)) \leq \epsilon + \frac{2}{k_0}\sqrt{k_1\epsilon_\Omega}] \nonumber \\
&\leq \frac{2k_1}{k_0}\sqrt{k_1\epsilon_\Omega} + k_1\epsilon \quad \text{(by Assumption~\ref{as:low-weight-around-opt})}
\end{align}
\end{ceqn}

As $\Pmbb(A\cap B) \leq min\{\Pmbb(A), \Pmbb(B)\}$, the inequality used in the second step is rather loose, but it shows the dependency on sufficiently small $\epsilon$. It could be independent of the tolerance $\epsilon$ depending on the $\Pmbb(\eta(X) - \oline\delta)$ or the sheer big value of $\epsilon$. Nevertheless, a similar result applies to the true negative rate. 
Since $\phi$ is 1-Lipschitz, we have that $|\phi(C)-\phi(C')|\leq 1\cdot \Vert C-C'\Vert$,
but $$\Vert C(\theta_0)-C(\theta')\Vert_\infty 
\leq \frac{2k_1}{k_0}\sqrt{k_1\epsilon_\Omega} + k_1\epsilon. $$

Hence, $|\phi(C_{\theta'}) - \phi(C_{\oline{\theta}})| \leq \sqrt{2}(\frac{2k_1}{k_0}\sqrt{k_1\epsilon_\Omega} + k_1\epsilon) + \epsilon_\Omega.$  Since the metrics are in $[0, 1]$, $\epsilon_\Omega \in [0, 1]$. Therefore, $\sqrt{\epsilon_\Omega} 
\geq \epsilon_\Omega$. This gives us the desired result.

\item We needed only, for part (ii), that the interval of possible values of $\theta'$ be at most $\epsilon$ to the target angle $\theta_0$. Ideally, this is obtained by making $\log_2(1/\epsilon)$ queries, but due to the region where oracle misreport its preferences, we can be off to the target angle $\theta_0$ by more than $\epsilon$. However, binary search will again put us back in the correct direction, once we leave the misreporting region. And this time, even if we are off to the target angle $\theta_0$, we will be closer than before. Therefore, for the interval of possible values of $\theta'$ to be at most $\epsilon$, we require at least $\log(\frac{1}{\epsilon})$ rounds of the algorithm, each of which is a constant number of pairwise queries.

\end{enumerate}\vspace*{-2em}
\end{proof}

\pb
\begin{proof}[Proof of Lemma~\ref{lem:lower-bound}] \textit{``Under our model, no algorithm can find the maximizer (minimizer) in fewer than~$O(\log\frac1\epsilon)$ queries.'' }

For any fixed $\epsilon$, divide the search space $\theta$ into bins of length $\epsilon$, resulting in $\ceil[\big]{\frac{1}{\epsilon}}$ classifiers. When the function evaluated on these classifiers is unimodal, and when the only operation allowed is pairwise comparison, the optimal worst case complexity for finding the argument maximum (of function evaluations) is $O(\log\frac1\epsilon)$ \cite{cormen2009introduction}, which is achieved by binary search. 
\end{proof}
\pb
\bprop\label{pr:sample-concentration-confusion} 
    Let $(y_1,x_1,h(x_1)),\,\dotsc,\,(y_n,x_n,h(x_n))$ be $n$ i.i.d.~samples from the joint distribution on $Y$, $X$, and $h(X)$. Then by H\"offding's inequality, 
	$$\Pmbb\left[\left|\tfrac1n\textstyle\sum_{i=1}^n\1[h_i=y_i=1] - TP(h)\right|\geq \epsilon \right]\leq 2e^{-2n\epsilon^2}.$$
	The same holds for the analogous estimator on TN.
\eprop
\bproof
Direct application of H\"offding's inequality.
\eproof

\pb
\begin{proof}[Proof of Theorem~\ref{thm:linear}] \textit{``Let $\varphi_{LPM} \ni \sphi = \smmbf$ be the true performance metric. Under Assumption~\ref{as:low-weight-around-opt}, given $\epsilon > 0$, LPM elicitation (Section~\ref{ssec:elicit_linear}) outputs a 
performance metric $\hphi = \hmmbf$, such that $\norm{\smmbf - \hmmbf}_\infty < \sqrt{2}\epsilon + \frac 2{k_0}\sqrt{2k_1\epsilon_\Omega}$.''}

We will show this for threshold classifiers, as in the statement of the Assumption~\ref{as:low-weight-around-opt}, but it is not difficult to extend the argument to the case of querying angles. (Involves a good bit of trigonometric identities...)

Recall, the threshold estimator $h_\delta$ returns positive if $\eta(x)\geq \delta$, and zero otherwise. Let $\oline\delta$ be the threshold which maximizes performance with respect to $\phi$, and $C_{\oline\delta}$ be its confusion matrix. 
For simplicity, suppose that $\delta'<\oline\delta$. Recall, from Assumption~\ref{as:low-weight-around-opt} that $\Pr[\eta(X)\in [\oline\delta-\frac{k_0}{2k_1}\epsilon,\,\oline\delta]]\leq k_0\epsilon/2$,
but $\Pr[\eta(X)\in[\oline\delta-\epsilon,\oline\delta]]\geq k_0\epsilon$, and therefore 
\[
	\Pmbb\Big[\eta(X)\in[\oline\delta-\epsilon,\oline\delta-\tfrac{k_0}{2k_1}\epsilon]\Big]\geq k_0\epsilon/2
\]
Denoting $\phi(C)=\langle\mmbf,C\rangle$, and recalling that $\oline\delta = m_{00}/(m_{11}+m_{00})$, expanding the integral,
we get
$$
\phi(C_{\oline\delta})-\phi(C_{\delta'})
$$
\small
\begin{align*}
	&=
    \int_{x:\delta'\leq \eta(x)\leq \oline\delta}\!\!\!\!\!\!\!\!\!\!\!\! [m_{00}(1-\eta(x))-m_{11}\eta(x)]\df \\
    &=\int_{x:\oline\delta - (\oline\delta - \delta')\leq \eta(x)\leq \oline\delta}\!\!\!\!\!\!\!\!\!\!\!\! [m_{00}(1-\eta(x))-m_{11}\eta(x)]\df \\
    &\geq\int_{x:\oline\delta - (\oline\delta - \delta')\leq \eta(x)\leq \oline\delta - \frac{k_0}{2k_1}(\oline\delta - \delta')}\!\!\!\!\!\!\!\!\!\!\!\! [m_{00}(1-\eta(x))-m_{11}\eta(x)]\df \\
\end{align*}
\begin{align*}
    &\geq[(m_{11} + m_{00})\big(\frac{-m_{00}}{m_{00} + m_{11}} + \frac{k_0}{2k_1}(\oline\delta - \delta')\big) + m_{00}] \times \\ &\qquad \int_{x:\oline\delta - (\oline\delta - \delta')\leq \eta(x)\leq \oline\delta - \frac{k_0}{2k_1}(\oline\delta - \delta')}\!\!\!\!\!\!\!\!\!\!\!\! \df \\
    &=[(m_{11} + m_{00}) \frac{k_0}{2k_1}(\oline\delta - \delta')] \times \\
    &\qquad \Pmbb[\oline\delta - (\oline\delta - \delta')\leq \eta(x)\leq \oline\delta - \frac{k_0}{2k_1}(\oline\delta - \delta')]\\
    &\geq \tfrac{k_0}2(\oline\delta-\delta') \cdot \tfrac{k_0}{2k_1}(\oline\delta-\delta')=\frac{k_0^2}{4k_1}(\oline\delta-\delta')^2.
\end{align*}
\normalsize
Similar results hold when $\delta'>\oline\delta$. 
Therefore, if we have $|\phi(\oline C)-\phi(C(\delta'))|<\epsilon_\Omega$, then we must have $|\oline\delta-\delta'|<\frac 2{k_0}\sqrt{
k_1\epsilon_\Omega}$. Thus, if we are in a regime where the oracle
is mis-reporting the preference ordering, it must be the case that
the thresholds are sufficiently close to the optimal threshold.

Again, as in the proof of Theorem~\ref{thm:quasi}, when the tolerance $\epsilon$ is small, our binary search closes in on a parameter $\theta'$ which has $\phi(C_{\delta_{\theta'}})$ within $\epsilon_\Omega$ of the optimum, but from the above discussion, this also implies that the search interval itself is close to the true value, and thus, the total error in the threshold is at most $\epsilon + \frac 2{k_0}\sqrt{k_1\epsilon_\Omega}$. Since $\oline\delta = m_{00}/(m_{11}+m_{00})$, this bound extends to the cost vector with a factor of $\sqrt2$, thus giving the desired result.

We observe that the above theorem actually provide bounds on the slope of the hyperplanes. Thus, the guarantees for LFPM elicitation follow naturally. It only requires that we recover the slope at the upper boundary and lower boundary correctly (within some bounds). This theorem provides those guarantees.   Algorithm~\ref{alg:grid-search} is independent of oracle queries and thus can be run with high precision, making the solutions of the two systems match.  
\end{proof}
\pb
\begin{proof}[Proof of Lemma~\ref{lem:sample-Cs-optimize-well}] \textit{``Let $h_{\theta}$ and $\hat h_{\theta}$ be two classifiers estimated using $\eta$ and $\hat\eta$, respectively. 
Further, let ${\oline{\theta}}$ be such that $h_{{\oline{\theta}}} = \argmax_{\theta}\phi(h_{\theta})$. Then
	${\Vert C(\hat h_{{\oline{\theta}}}) - C(h_{{\oline{\theta}}}) \Vert_\infty=O( \Vert {\hat\eta}_n-\eta\Vert_\infty})$.'' }
	
Suppose the performance metric of the oracle is characterized by the parameter $\oline\theta$. Recall the Bayes optimal classifier would be $h_{\oline{\theta}} = \1 [\eta\geq \oline{\delta}]$. Let us assume we are given a classifier $\hhat_{\oline{\theta}} = \1 [\hat\eta\geq \oline{\delta}]$. Notice that the optimal threshold $\oline\delta$ is the property of the metric and not the classifier or $\eta$. We want to bound the difference in the confusion matrices for these two classifiers. Notice that, by  Assumption~\ref{as:sup-norm-convergence}, we can take $n$ sufficiently large so that $\Vert \eta-\hat\eta_n\Vert_\infty$ is arbitrarily small. Consider the quantity
\begin{ceqn}
\begin{align*}
	TP(h_{\oline{\theta}}) - TP(\hat h_{\oline{\theta}}) &= \int_{\eta\geq \oline{\delta}} \!\!\!\!\!\!\!\eta \df  -̥
    \int_{\hat\eta\geq \oline{\delta}} \!\!\!\!\!\!\!\eta \df.
\end{align*}
\end{ceqn}
Now the maximum loss in the above quantity can occur when, in the region where the classifiers' predictions differ, there $\hat\eta$ is less than $\eta$ with the maximum possible difference. This is equal to
\begin{align*}
    &\int\limits_{x:\oline\delta \leq \eta(x) \leq \oline\delta + \Vert\eta - \hat\eta\Vert_\infty} \!\!\!\!\!\!\!\eta \df \\
    &\leq \Pmbb[\oline\delta \leq \eta(X) \leq \oline\delta + \Vert\eta - \hat\eta\Vert_\infty] \\
    &\leq k_1\Vert\eta - \hat\eta\Vert_\infty. \qquad \text{(by Assumpition~\ref{as:low-weight-around-opt})}
\end{align*}

Similarly, we can look at the maximum gain in the following quantity.

\begin{align*}
	 TP(\hat h_{\oline{\theta}}) - TP(h_{\oline{\theta}}) &= 
    \int_{\hat\eta\geq \oline{\delta}} \!\!\!\!\!\!\!\eta \df - 
    \int_{\eta\geq \oline{\delta}} \!\!\!\!\!\!\!\eta \df  
    \label{eq:gain}
\end{align*}

Now the maximum gain in the above quantity can occur when, in the region where the classifiers' predictions differ, there $\hat\eta$ is greater than $\eta$ with the maximum possible difference. This is equal to
\begin{ceqn}
\begin{align*}
    &\int\limits_{x: \oline\delta - \Vert\eta - \hat\eta\Vert_\infty\leq\eta(x)\leq \oline\delta} \!\!\!\!\!\!\!\eta \df \\
    &\leq \Pmbb[\oline\delta - \Vert\eta - \hat\eta\Vert_\infty\leq \eta(X) \leq \oline\delta] \\
    &\leq k_1\Vert\eta - \hat\eta\Vert_\infty. \qquad \text{(by Assumpition~\ref{as:low-weight-around-opt})}
\end{align*}
\end{ceqn}
Hence, $$|TP(\hat h_{\oline{\theta}}) - TP(h_{\oline{\theta}})| \leq k_1\Vert\eta - \hat\eta\Vert_\infty.$$

Similar arguments apply for $TN$, which gives us the desired result.
\end{proof}
\pb

\section{Monotonically Decreasing Case}\label{appendix:decreasing}

Even if the oracle's metric is monotonically decreasing in \emph{TP} and \emph{TN}, we can figure out the supporting hyperplanes at the maximizer and the minimizer. It would require to pose one query $\Omega(C^*_{\pi/4}, C^*_{5\pi/4})$.  The response from this query determines whether we want to search over $\partial\Ccal_+$ or $\partial\Ccal_-$ and apply Algorithms \ref{alg:linear} and \hyperlink{alg:quasiconvex}{2} accordingly. In fact, if $C^*_{\pi/4} \prec C^*_{5\pi/4}$, then the metric is monotonically decreasing, and we search for the maximizer on the lower boundary $\partial\Ccal_-$. Similarly if the converse holds, then we search over $\partial\Ccal_+$ as discussed in the main paper. 
\end{appendices}
	
\end{document}